\documentclass{article}

\usepackage{pdfpages}
\usepackage{etoolbox}
\usepackage{authblk}

\usepackage[utf8]{inputenc} 
\usepackage[T1]{fontenc} 
\usepackage{hyperref} 
\usepackage{url} 
\usepackage{booktabs} 
\usepackage{amsfonts} 
\usepackage{nicefrac} 
\usepackage{microtype} 
\usepackage{xspace} \title{Fair Algorithms for Infinite and Contextual Bandits}
\usepackage{wrapfig}

\usepackage{booktabs, tabularx}
\usepackage{tikz,fancyhdr,cancel,enumerate,booktabs,setspace,pgf,tikz,fancyhdr,color}
\usetikzlibrary{patterns,arrows,decorations.pathreplacing}

\usepackage{amssymb,amsmath,amsthm}

\newtheorem{lemma}{Lemma}
\newtheorem{theorem}{Theorem}

\usepackage{algorithm}
\usepackage[noend]{algpseudocode}
\usepackage{xcolor}

\usepackage{natbib}
\theoremstyle{definition}
\newtheorem{definition}{Definition}
\theoremstyle{remark}
\newtheorem{remark}{Remark}
\usepackage{fullpage}
\usepackage{makecell}

\newcommand{\ridgefair}{\textsc{RidgeFair}\xspace}
\newcommand{\xt}{\ensuremath{x_t}\xspace}
\newcommand{\yt}{\ensuremath{y_t}\xspace}
\newcommand{\xtp}[1]{\ensuremath{x_{t,#1}}\xspace}
\newcommand{\xs}[1]{\ensuremath{x^*_{#1}}\xspace}
\newcommand{\xts}{\xs{t}}
\newcommand{\xti}{\xtp{i}\xspace}
\newcommand{\ytp}[1]{\ensuremath{y_{t,#1}}\xspace}
\newcommand{\yti}{\ensuremath{\ytp{i}}\xspace}
\newcommand{\et}[1]{\ensuremath{\eta_{#1}}\xspace}
\newcommand{\ett}{\ensuremath{\et{t}}\xspace}
\newcommand{\Ypt}{\ensuremath{Y_t}\xspace}
\newcommand{\Ypti}{\ensuremath{Y_{t,i}}\xspace}
\newcommand{\Yp}[1]{\ensuremath{Y_{#1}}\xspace}
\newcommand{\Yt}{\ensuremath{Y_t}\xspace}

\newcommand{\Y}[1]{\ensuremath{Y_{#1}}\xspace}
\newcommand{\Xt}{\ensuremath{X_t}\xspace}
\newcommand{\D}{\ensuremath{D}\xspace}
\newcommand{\G}{\ensuremath{\mathcal{G}}\xspace}
\newcommand{\X}[1]{\ensuremath{X_{#1}}\xspace}
\newcommand{\bXt}{\ensuremath{{\bf X}_t}\xspace}
\newcommand{\bX}[1]{\ensuremath{{\bf X}_{#1}}\xspace}
\newcommand{\bYt}{\ensuremath{{\bf Y}_t}\xspace}
\newcommand{\bY}[1]{\ensuremath{{\bf Y}_{#1}}\xspace}
\newcommand{\bht}{\ensuremath{\hat{\beta}_t}\xspace}

\renewcommand{\b}{\ensuremath{\beta}\xspace}
\newcommand{\cbxt}{\ensuremath{\langle \b, \xt \rangle}\xspace}
\newcommand{\R}{\mathbb{R}\xspace}
\newcommand{\Ct}{\ensuremath{C_t}\xspace}
\newcommand{\C}[1]{\ensuremath{C_{#1}}\xspace}
\newcommand{\Pt}{\ensuremath{P_t}\xspace}
\renewcommand{\P}[1]{\ensuremath{P_{#1}}\xspace}
\newcommand{\Omt}{\Omega_t}
\newcommand{\xths}{\hat{x}_{t}^*}
\newcommand{\xtpr}{\ensuremath{x_{t'}}}
\newcommand{\xtpri}[1]{\ensuremath{x_{t',#1}}\xspace}
\newcommand{\ytpr}{\ensuremath{y_{t'}}}
\newcommand{\etapr}{\ensuremath{\eta_{t'}}}

\newcommand{\Pst}{\ensuremath{P_{*,t}}\xspace}
\newcommand{\Ps}[1]{\ensuremath{P_{*,#1}}\xspace}
\newcommand{\A}{\ensuremath{\mathcal{A}}\xspace}
\newcommand{\E}[1]{\ensuremath{\mathbb{E}\left[#1 \right]}\xspace}
\newcommand{\Ex}[2]{\mathbb{E}_{#1}\lbrack #2 \rbrack}
\newcommand{\Ep}{\mathbb{E}}
\newcommand{\poly}{\textrm{poly}\xspace}
\newcommand{\LB}{1-bandit\xspace}
\newcommand{\MLB}{m-bandit\xspace}
\newcommand{\UB}{k-bandit\xspace}
\newcommand{\rew}{\textrm{Rew}}
\newcommand{\pp}[1]{\pi_{#1}}
\newcommand{\ppt}{\pp{t}}
\renewcommand{\Pr}[2]{\mathbb{P}_{#1}\left[#2\right]}
\newcommand{\vecdot}[2]{\langle #1, #2 \rangle}
\newcommand{\eeGap}{\textsc{FairGap}\xspace}
\newcommand{\aeeGap}{\textsc{ApproxFairGap}}
\newcommand{\toptwo}{\textsc{TopTwo}\xspace}
\newcommand{\uar}{\textrm{UAR}}

\newcommand{\gap}{\Delta_{\text{gap}}}
\newcommand{\prior}{\tau}
\newcommand{\iht}{{\hat{i}_t\xspace}}
\newcommand{\xht}{{\hat{x}_t\xspace}}
\newcommand{\ist}{{i_{*,t}\xspace}}
\newcommand{\rti}{r_{t,i}\xspace}
\newcommand{\rt}[1]{{r_{t,#1}}\xspace}
\newcommand{\rflb}{\ensuremath{\ridgefair_1}\xspace}
\newcommand{\rfmlb}{\ensuremath{\ridgefair_m}\xspace}
\newcommand{\rfub}{\ensuremath{\ridgefair_{\leq k}}\xspace}
\newcommand{\rfj}{\ensuremath{\ridgefair_j}\xspace}

\newcommand{\ey}[1]{\hat{y}_{t,#1}\xspace}
\newcommand{\eyi}{\ey{i}\xspace}
\newcommand{\w}[1]{w_{t,#1}\xspace}
\newcommand{\wt}{w_t\xspace}
\newcommand{\wti}{w_{t,i}\xspace}
\renewcommand{\xti}{x_{t,i}\xspace}
\newcommand{\low}[2]{\ell_{#1, #2}\xspace}
\newcommand{\up}[2]{u_{#1, #2}\xspace}
\newcommand{\interval}[2]{\left[\low{#1}{#2}, \up{#1}{#2}\right]\xspace}
\newcommand{\pick}{\textsc{Pick}\xspace}
\newcommand{\argmax}{\textrm{argmax}\xspace}
\renewcommand{\S}[1]{S_{#1}\xspace}
\newcommand{\ettb}{{\bf \eta}_t}
\newcommand{\regret}[1]{\textsc{Regret}\left(#1\right)\xspace}
\newcommand{\Xtrans}{{\bf X}^T}
\newcommand{\Xttrans}{{{\bf X}_{t}}^T}
\newcommand{\Vt}{\bar{V}_t}
\newcommand{\VT}{\bar{V}^T}
\newcommand{\Vti}{(\bar{V}_t)^{-1}}

\author[1]{Matthew Joseph \thanks{\texttt{majos@cis.upenn.edu}}}
\author[1]{Michael Kearns \thanks{\texttt{mkearns@cis.upenn.edu}}}
\author[1]{Jamie Morgenstern \thanks{\texttt{jamiemor@cis.upenn.edu}}}
\author[2]{Seth Neel \thanks{\texttt{sethneel@wharton.upenn.edu}} }
\author[1]{Aaron Roth \thanks{\texttt{aaroth@cis.upenn.edu}} }
\affil[1]{Computer and Information Science, University of Pennsylvania}
\affil[2]{Statistics Department, The Wharton School, University of Pennsylvania}

\begin{document}
\date{}
\maketitle

\begin{abstract}
  We study fairness in linear bandit problems. Starting from the
  notion of meritocratic fairness introduced in~\citet{JKMR16}, we
  carry out a more refined analysis of a more general problem,
  achieving better performance guarantees with fewer modelling
  assumptions on the number and structure of available choices as well
  as the number selected. We also analyze the previously-unstudied
  question of fairness in infinite linear bandit problems, obtaining
  instance-dependent regret upper bounds as well as lower bounds
  demonstrating that this instance-dependence is necessary. The result
  is a framework for meritocratic fairness in an online linear setting
  that is substantially more powerful, general, and realistic than the
  current state of the art.
\end{abstract}

\section{Introduction}\label{sec:intro}
The problem of repeatedly making choices and learning from choice
feedback arises in a variety of settings, including granting loans,
serving ads, and hiring. Encoding these problems in a \emph{bandit}
setting enables one to take advantage of a rich body of existing
bandit algorithms. UCB-style algorithms, for example, are guaranteed
to yield no-regret policies for these problems.

\citet{JKMR16}, however, raises the concern that these no-regret
policies may be \emph{unfair}: in some rounds, they will choose
options with lower expected rewards over options with higher expected
rewards, for example choosing less qualified job applicants over more
qualified ones. Consider a UCB-like algorithm aiming to hire all
qualified applicants in every round.  As time goes on, any no-regret
algorithm must behave unfairly for a vanishing fraction of rounds, but
the total number of \emph{mistreated} people -- in hiring, people who
saw a less qualified job applicant hired in a round in which they
themselves were not hired -- can be large (see
Figure~\ref{fig:mistreatment}).


\begin{wrapfigure}{r}{0.38\textwidth}
\begin{center}
\vspace{-22pt}
\includegraphics[width=0.38\textwidth]
{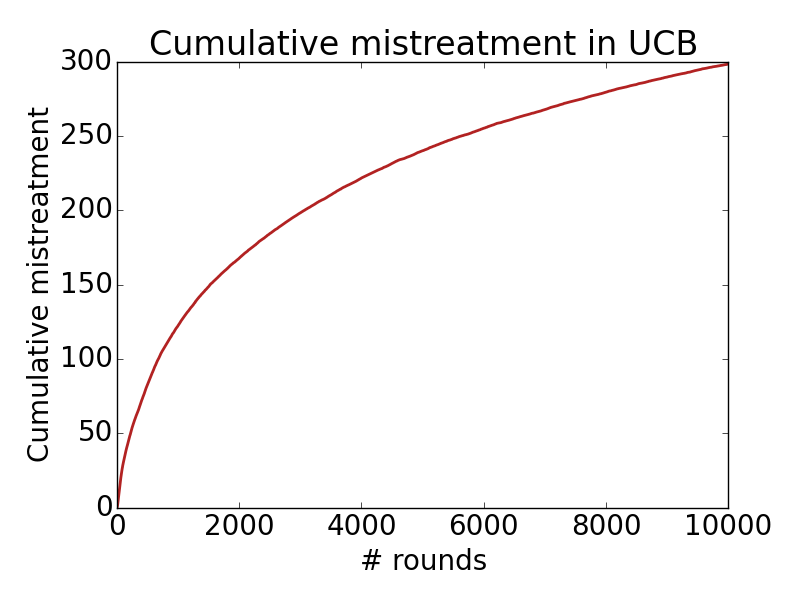}
\vspace{-22pt}
\caption{Cumulative mistreatments for UCB. See
  Section~\ref{subsec:experiments} in supplement for details and
  additional experimental evaluation of the structure of
  mistreatment.}
\end{center}
\label{fig:mistreatment}
\vspace{-91pt}
\end{wrapfigure}

~\citet{JKMR16} then design no-regret algorithms which minimize mistreatment
and are fair in the following
sense: their algorithms (with high probability) never at any round
place higher selection probability on a less qualified applicant than
on a more qualified applicant. However, their analysis assumes that
there are $k$
well-defined groups, each with its own mapping from features to
expected rewards; at each round exactly one individual from each group
arrives; and exactly one individual is chosen in each
round. In the hiring setting, this equates to assuming that a company
receives one job applicant from each group and must hire exactly one
(rather than $m$ or all qualified applicants) introducing an
unrealistic element of competition and unfairness both between applicants and
between groups.

The aforementioned assumptions are unrealistic in many practical
settings; our work shows they are also \emph{unnecessary}.
Meritocratic fairness can be defined without reference to groups, and
algorithms can satisfy the strictest form of meritocratic fairness
without any knowledge of group membership. Even without this
knowledge, we design algorithms which will be fair with respect to
\emph{any} possible group structure over individuals.  In
Section~\ref{sec:model}, we present this general definition of
fairness. The definition further allows for the number of individuals
arriving in any round to vary, and is sufficiently flexible to apply
to settings where algorithms can select $m \in [k]$ individuals in
each round. By virtue of the definition making no reference to groups,
the model makes no assumptions about how many individuals arriving at
time $t$ belong to any group. A company can then consider a large pool
of applicants, not necessarily stratified by race or gender, with an
arbitrary number of candidates from any one of these populations, and
hire one or $m$ or even every qualified applicant.

We then present a framework for designing meritocratically fair online linear
contextual bandit algorithms. In Section~\ref{sec:finite}, we show how to
design fair algorithms when at most some finite number $k$ of individuals
arrives in any round (the linear contextual bandits
problem~\citep{abe2003reinforcement,auer2002using}), as well as when $m$ 
individuals may be chosen in each round (the ``multiple play" introduced and
studied absent fairness in~\citet{AVW87}).  We therefore study a much more
general model than~\citep{JKMR16} and, in Section~\ref{sec:finite},
substantially improve upon their black-box regret guarantees for linear bandit
problems using a technical analysis specific to the linear setting.

However, these regret bounds still scale (polynomially) with $k$, the
maximum number of individuals seen in any given round. This may be
undesirable for large $k$, thus motivating the investigation of fair
algorithms for the \emph{infinite} bandit setting (the online linear
optimization with bandit feedback
problem~\cite{flaxman2005online}).\footnote{We note that both the
  finite and infinite settings have infinite numbers of potential
  candidates: the difference arises in how many choices an algorithm
  has in a given round.} In Section~\ref{sec:infinite} we provide such
an algorithm via an adaptation of our general confidence
interval-based framework that takes advantage of the fact that optimal solutions to linear programs must be \emph{extreme points} of the feasible region. We then prove, subject to certain
assumptions, a regret upper bound that depends on $\gap$, an
instance-dependent parameter based on the distance between the best
and second-best extreme points in a given choice set. 

In Section~\ref{sec:lower_bound1} we show that this instance dependence
is almost tight by exhibiting an infinite choice set satisfying our
assumptions for which \emph{any} fair algorithm must incur regret
dependent polynomially on $\gap$, separating this setting from the online linear optimization setting absent a fairness constraint. Finally, we justify our assumptions
on the choice set by in Section~\ref{sec:lower_bound2} exhibiting a
choice set that both violates our assumptions and admits \emph{no}
fair algorithm with nontrivial regret guarantees. A condensed
presentation of our methods and results appears in
Figure~\ref{fig:results}.

Finally, we note that our algorithms share an overarching logic for reasoning
about fairness. These algorithms all satisfy fairness by
\emph{certifying optimality}, never giving preferential treatment to
$x$ over $y$ unless the algorithm is \emph{certain} that $x$ has
higher reward than $y$. The algorithms accomplish this by computing
confidence intervals around the estimated rewards for individuals. If
two individuals have overlapping confidence intervals, we say they are
\emph{linked}; if $x$ can be reached from $y$ using a sequence of
linked individuals, we say they are \emph{chained}.

\begin{figure}[h]
\vspace{10pt}
 \begin{center}
  \begin{tabular}{| p{2cm} |p{1.5cm}|p{3.3cm}|p{2cm}|p{3cm}| }\hline
\thead{  \# selected\\ each round} &\makecell[l]{ \# options\\ each round} & Technique & Notes & \thead{Regret} \\\hline
 \thead{ Exactly \\ $j \leq k$} & $\leq k$ &\makecell*[l]{Play all of chains in
 \\ descending order,\\ randomizing over \\ last chain as necessary \\ to pick
 exactly $j$} & \makecell[l]{Requires\\ randomness} & $\tilde{O}\left(dkj\sqrt{T}\right)$  \\ \hline
 \thead{Unconstrained} & $\leq k$ & Select all in every chain with highest UCB $> 0$
& Deterministic & $\tilde{O}\left(dk^2\sqrt{T}\right)$  \\ \hline
    Exactly $1$  & \thead{$\infty$\\ bounded\\ convex set\\ $\gap > 0$} & \makecell[l]{Play uniquely best point\\ or UAR from entire set} & \makecell[l]{Requires\\ randomness}&\thead{ $\tilde{O}\left(c \cdot \log(T)/\gap^2\right)$\\ $\tilde{\Omega}(1/\gap)$ \\ $\Omega(T)$ for $\gap = 0$}   \\ \hline
\end{tabular}
\end{center}
\caption{A description of various settings in which our framework
  provides fair algorithms. In all cases, fairness can be imposed only
  across pairs for any partitioning of the input space; the bounds
  here assume they bind across all pairs, and are therefore worst-case
  upper bounds. See Section~\ref{sec:infinite} for a complete
  explanation of the distribution-dependent constant $c$ in the regret
  bound for the infinite case.}
\label{fig:results}
\vspace{-10pt}
\end{figure}

\subsection{Related Work and Discussion of Our Fairness Definition}\label{sec:related}
Fairness in machine learning has seen substantial recent growth as a
subject of study, and many different definitions of fairness exist. We
provide a brief overview here; see e.g.~\citet{BHJKR17}
and~\citet{CPFGH17} for detailed descriptions and comparisons of these
definitions.

Many extant fairness notions are predicated on the existence of
\emph{groups}, and aim to guarantee that certain groups are not
unequally favored or mistreated. In this vein,~\citet{HPS16}
introduced the notion of \emph{equality of opportunity}, which
requires that a classifier's predicted outcome should be independent
of a protected attribute (such as race) conditioned on the true
outcome, and they and~\citet{WGOS17} have studied the feasibility and
possible relaxations thereof. Similarly, ~\citet{ZVRG17} analyzed an
equivalent concurrent notion of (un)fairness they call \emph{disparate
  mistreatment}. Separately,~\citet{KMR17} and~\citet{C17} showed that
different notions of group fairness may (and sometimes must) conflict
with one another.

This paper, like~\citet{JKMR16}, departs from the work above in a
number of ways. We attempt to capture a particular notion of
\emph{individual} and \emph{weakly meritocratic} fairness that holds
\emph{throughout the learning process}. This was inspired
by~\citet{DHPRZ12}, who suggest fair treatment equates to treating
``similar'' people similarly, where similarity is 
defined with respect to an assumed pre-specified task-specific metric.  Taking
the fairness formulation of~\citet{JKMR16} as our starting point, our
definition of fairness does not promise to correct for past inequities
or inaccurate or biased data. Instead, it assumes the existence of an
accurate mapping from features to true quality for the task at
hand\footnote{~\citet{FSV16} provide evidence that providing fairness
  from bias-corrupted data is quite difficult.}  and promises fairness
while learning and using this mapping in the following sense: any
\emph{individual} who is currently more qualified (for a job, loan, or
college acceptance) than another individual will always have at least
as good a chance of selection as the less qualified individual.

The one-sided nature of this guarantee, as well as its formulation in terms of
quality, leads to the name \emph{weakly meritocratic} fairness. Weakly
meritocratic fairness may then be interpreted as a minimal guarantee of
fairness: an algorithm satisfying our fairness definition cannot favor a worse
option but is not required to favor a better option. In this sense our fairness
requirement encodes a necessary variant of fairness rather than a completely
sufficient one. This makes our upper bounds (Sections~\ref{sec:finite}
and~\ref{sec:infinite}) relatively weaker and our lower bounds
(Sections~\ref{sec:lower_bound1} and~\ref{sec:lower_bound2}) relatively
stronger.

We additionally note that our fairness guarantees require fairness
\emph{at every step of the learning process}. We view this as an
important point, especially for algorithms whose learning processes
may be long (or even continuous).  Furthermore, while it may seem
reasonable to relax this requirement to allow a small fraction of
unfair steps, it is unclear how to do so without enabling
discrimination against a correspondingly small population.

Finally, while our fairness definition draws from~\citet{JKMR16}, we
work in what we believe to be a significantly more general and
realistic setting. In the finite case we allow for a variable number
of individuals in each round from a variable number of groups and also
allow selection of a variable number of individuals in each round,
thus dropping several assumptions from~\citet{JKMR16}. We also analyze
the previously unstudied topic of fairness with infinitely many
choices.

\section{Model}\label{sec:model}

Fix some $\b\in [-1,1]^d$, the underlying linear coefficients of our
learning problem, and $T$ the number of rounds. For each $t\in [T]$,
let $\Ct \subseteq \D = [-1,1]^d$ denote the set of available choices
in round $t$. We will consider both the ``finite'' action case, where
$|\Ct|\leq k$, and the infinite action case. An algorithm \A, facing
choices \Ct, picks a subset $\Pt\subseteq \Ct$, and for each
$\xt\in \Pt$, \A observes reward $\yt \in [-1,1]$ such that
$\E{\yt} = \cbxt$, and the distribution of the noise
$\ett = \yt - \cbxt$ is sub-Gaussian (see
Section~\ref{sec:subgaussian} for a definition of sub-Gaussian).

Refer to all observations in round $t$ as $\Ypt \in [-1,1]^{|\Pt|}$
where $\Ypti = \yti$ for each $\xti\in\Pt$.  Finally, let
$\bXt = [\X{1}; \ldots; \X{t}], \bYt = [\Y{1}; \ldots; \Y{t}]$ refer
to the design and observation matrices at round $t$.

We are interested in settings where an algorithm may face size
constraints on $\Pt$. We consider three cases: the standard linear
bandits problem ($|\Pt| = 1$), the multiple choice linear bandits
problem ($|\Pt| = m$), and the heretofore unstudied (to the best of
the authors' knowledge) case in which the size of $\Pt$ is
unconstrained.  For short, we refer to these as \LB, \MLB, and
\UB. 

\paragraph{Regret}
The notion of regret we will consider is that of pseudo-regret.
Facing a sequence of choice sets $\C{1}, \ldots, \C{T}$, suppose $\A$
chooses sets $\P{1}, \ldots, \P{T}$.\footnote{If these are randomized
  choices, the randomness of $\A$ is incorporated into the expected
  value calculations.}  Then, the expected reward of $\A$ on this
sequence is
$\rew(\A) = \E{\sum_{t\in [T]}\left[\sum_{\xt\in \Pt}\yt\right]}$.

Refer to the sequence of feasible choices\footnote{We assume these
  have the appropriate size for each problem we consider: singletons
  in the \LB problem, size at most $m$ in the \MLB problem, and
  arbitrarily large in the \UB problem.}  which maximizes expected
reward as $\Ps{1}\subseteq \C{1}, \ldots, \Ps{T} \subseteq \C{T}$,
defined with full knowledge of $\b$.

Then, the {\bf pseudo-regret} of $\A$ on any sequence is defined as

\[ \rew( \Ps{1}, \ldots, \Ps{T}) - \rew(\A) = R(T).\]

The {\bf pseudo-regret} of $\A$ refers to the maximum pseudo-regret
$\A$ incurs on any sequence of choice sets and any $\b\in [-1,1]^d$.
If $R(T) = o(T)$, then $\A$ is said to be {\bf no-regret}. If, for any
input parameter $\delta > 0$, $R(T)$ upper-bounds the expectation of
the rewards of the sequence chosen by $\A$ with probability
$1-\delta$, then we call this a \emph{high-probability} regret bound
for $\A$.
\paragraph{Fairness}
Consider an algorithm $\A$, which chooses a sequence of
\emph{probability distributions} $\pp{1}, \pp{2}, \ldots, \pp{T}$ over
feasible sets to pick, $\pp{t} \in \Delta(2^{\Ct})$. Note that
distribution $\pp{t}$ depends upon $\C{1}, \ldots, \C{t}$, the choices
$\P{1}, \ldots, \P{t-1}$, and $\Yp{1}, \ldots, \Yp{t-1}$.

We now give a formal definition of fairness of an algorithm for the
\LB, \MLB, and \UB problems.  We adapt our fairness definition
from~\citet{JKMR16}, generalizing from discrete distributions over
finite action sets to mixture distributions over possibly infinite
action sets.  We slightly abuse notation and refer to the probability
density and mass functions of an element $x\in \Ct$: this refers to
the marginal distribution of $x$ being chosen (namely, the probability
that $x$ belongs to the set picked according to the distribution
$\pp{t}$).

\begin{definition}[Weakly Meritocratic Fairness]\label{def:fair}
  We say that an algorithm $\A$ is \emph{weakly meritocratic} if, for
  any input $\delta \in (0,1]$ and any $\b$, with probability at least
  $1-\delta$, at every round $t$, for every $x, x' \in \Ct$ such that
  $\vecdot{\b}{x} \geq \vecdot{\b}{x'}$: 
  \begin{itemize}
  \item If $\pi_t$ is a discrete distribution: For
    $g_{t}(x) = \pi_{t}(x)$ (the probability mass function)
    \[g_{t}(x) \geq g_{t}(x').\]

  \item If $\pi_t$ is a continuous distribution: For
    $g_{t}(x) = f_{t}(x)$ (the probability density function)
    \[g_{t}(x) \geq g_{t}(x').\]

  \item If $\pi_t$ can be written as a mixture distribution:
  $\sum_{i} \alpha_i \pi_{ti}, \sum_i \alpha_i = 1$, such that each constituent distribution
   $\pi_{ti} \in \Delta(2^{\Ct})$ is either discrete or
  continuous and satisfies one of the above two conditions.
  \end{itemize}
  For brevity, as consider only this fairness notion in this paper, we
  will refer to weakly meritocratic fairness as ``fairness''.  We say
  $\A$ is {\bf round-fair} at time $t$ if $\pi_t$ satisfies the above
  conditions.
\end{definition}

This definition can be easily generalized over any partition $\G$ of
$\D$, by requiring this weak monotonicity hold \emph{only for pairs
  $x, x'$ belonging to different elements of the partition $G, G'$}.
The special case above of the singleton partition is the most
stringent choice of partition. We focus our analysis on the singleton
partition as a minimal worst-case framework, but this model easily
relaxes to apply only across groups, as well as to only requiring
``one-sided'' monotonicity, where monotonicity is required only for
pairs where the more qualified member belongs to group $G$ rather than
$G'$.

\begin{remark}
  In the \UB setting, Definition~\ref{def:fair} can be simplified to
  require, with probability $1-\delta$ over its observations, an
  algorithm \emph{never} select a less-qualified individual over
  more-qualified one in any round, and can be satisfied by
  deterministic algorithms.
\end{remark}
\section{Finite Action Spaces: Fair Ridge Regression}\label{sec:finite}

In this section, we introduce a family of fair algorithms for linear
\LB, \MLB, and the (unconstrained) \UB problems. Here, an algorithm sees a slate of at most $k$ distinct individuals each round
and selects some subset of them for reward and observation. This
allows us to encode settings where an algorithm repeatedly observes a
new pool of $k$ individuals, each represented by a vector of $d$
features, then decides to give some of those individuals loans based
upon those vectors, observes the quality of the individuals to whom
they gave loans, and updates the selection rule for loan
allocation. The regret of these algorithms will scale polynomially in
$k$ and $d$ as the algorithm gets tighter estimates of $\b$.

All of the algorithms are based upon the following template. They
maintain an estimate $\bht$ of $\b$ from observations, along with
confidence intervals around the estimate. They use $\bht$ to estimate
the rewards for the individuals on day $t$ and the confidence
interval around $\bht$ to create a confidence interval around each of
these estimated rewards. Any two individuals whose intervals overlap
on day $t$ will picked with the same probability by the algorithm.
Call any two individuals whose intervals overlap on day $t$ \emph{linked}, and
any two individuals belonging to the transitive closure of the linked
relation \emph{chained}. Since any two linked individuals will chosen
with the same probability, any two chained individuals will also be
chosen with the same probability.

An algorithm constrained to pick exactly $m\in [k]$ individuals
each round will pick them in the following way. Order the chains by
their highest upper confidence bound. In that order, select all
individuals from each chain (with probability $1$ while that results
in taking fewer than $m$ individuals. When the algorithm arrives at
the first chain for which it does not have capacity to accept every
individual in the chain, it selects to fill its capacity uniformly at
random from that chain's individuals.  If the algorithm can pick any
number of individuals, it will pick all individuals chained to any
individual with positive upper confidence bound.

We now present the regret guarantees for fair \LB, \MLB, and \UB using
this framework.

\begin{theorem}\label{thm:regret-finite}
  Suppose, for all $t$, $\ett$ is $1$-sub-Gaussian,
  $\Ct \subseteq [-1,1]^d$, and $||\xt||_2 \leq 1$ for all
  $\xt\in\Ct$, and $||\b|| \leq 1$. Then, \rflb, \rfmlb, and \rfub are
  fair algorithms for the \LB, \MLB, and \UB problems,
  respectively. With probability $1-\delta$, for $j \in \{1, m, k\}$,
  the regret of \rfj is
\[R(T) = O\left(d k j \sqrt{T}\log\left(\frac{T}{\delta}\right)\right) =
\tilde{O}(dkj\sqrt{T}).\]
\end{theorem}

We pause to compare our bound for \LB to that found in~\citet{JKMR16}.
Their work supposes that each of $k$ groups has an independent
$d$-dimensional linear function governing its reward and provides a
fair algorithm regret upper bound of
$\tilde{O}\left(T^{\frac{4}{5}}k^{\frac{6}{5}}d^{\frac{3}{5}},
  k^3\right)$.
To directly encode this setting in ours, one would need to use a
single $dk$-dimensional linear function, yielding a regret bound of
$\tilde{O}\left(dk^2\sqrt{T}\right)$. This is an improvement on their upper
bound for all values of $T$ for which the bounds are non-trivial (recalling that the bound
from~\citet{JKMR16} becomes nontrivial for $T > d^3k^6$, while the
bound here becomes nontrivial for $T > d^2k^4$). We also briefly
observe that~\rfub satisfies an additional ``fairness'' property: with high
probability, it always selects \emph{every} available individual with
positive expected reward.

Each of these algorithms will use $\ell_2$-regularized least-squares
regressor to estimate $\b$. Given a design matrix $\bX{}$, response
vector $\bY{}$, and regularization parameter $\gamma \geq 1$ this is
of the form
$\hat{\beta} = (\Xtrans \bX{} + \gamma I)^{-1} \Xtrans \bY{}$.  Valid
confidence intervals (that contain $\beta$ with high probability) are
nontrivial to derive for this estimator (which might be biased); to
construct them, we rely on martingale matrix concentration
results~\citep{APS11}.

We now sketch how the proof of Theorem~\ref{thm:regret-finite}
proceeds, deferring a full proof (of this and all other results in
this paper) and pseudocode to the supplementary materials.
We first establish that,
with probability $1-\delta$, for all rounds $t$, for all
$\xti\in \Ct$, that $\yti \in \interval{t}{i}$ (i.e. that the
confidence intervals being used are valid).  Using this fact, we
establish that the algorithm is fair. The algorithm plays any two
actions which are linked with equal probability in each round, and any
action with a confidence interval above another action's confidence
interval with weakly higher probability. Thus, if the payoffs for the
actions lie anywhere within their confidence intervals, \ridgefair is
fair, which holds as the confidence intervals are valid.

Proving a bound on the regret of \ridgefair requires some non-standard
analysis, primarily because the widths of the confidence intervals
used by the algorithm do not shrink uniformly. The sum of the widths
of the intervals of our \emph{selected} (and therefore observed)
actions grows sublinearly in $t$. UCB variants, by virtue of playing
an action $a$ with highest upper confidence bound, have regret in
round $t$ bounded by $a$'s confidence interval width. \ridgefair,
conversely, suffers regret equal to the \emph{sum} of the confidence
widths of the chained set, while only receiving feedback for the action
it actually takes. We overcome this obstacle by relating the sum of
the confidence interval widths of the linked set to the sum of the
widths of the selected actions.

\section{Fair algorithms for convex action sets}\label{sec:infinite}
In this section we analyze linear bandits with infinite choice sets in the
familiar \LB setting.\footnote{Note that no-regret guarantees are in general
impossible for infinite choice sets in \MLB and \UB settings, since the
continuity of the infinite choice sets we consider makes selecting multiple
choices while satisfying fairness impossible without choosing uniformly at
random from the entire set.} We provide a fair algorithm with an
instance-dependent sublinear regret bound for infinite choice sets --
specifically convex bodies -- below.  In Section~\ref{sec:lower_bound1} we
match this with lower bounds showing that instance dependence is an
unavoidable cost for fair algorithms in an infinite setting.

A naive adaptation of \ridgefair to an infinite setting requires maintenance of
infinitely many confidence intervals and is therefore impractical. We instead
assume that our choice sets are convex bodies and exploit the resulting
geometry: since our underlying function is linear, it is maximized
at an \emph{extremal} point. This simplifies the problem, since we need only
reason about the relative quality of extremal points. The relevant quantity is
$\gap$, a notion adapted from~\citet{DHK08} that denotes the difference in
reward between the best and second-best extremal
points in the choice set. When $\gap$ is large it is easier to confidently
identify the optimal choice and select it deterministically without violating
fairness. When $\gap$ is small, it is more difficult to determine which of the
top two points is best -- and since deterministically selecting the wrong one
violates fairness for any points infinitesimally close to the true best point,
we are forced to play randomly from the entire choice set.

Our resulting fair algorithm, \eeGap, proceeds as follows: in each round it
uses its current estimate of $\b$ to construct confidence intervals around the
two choices with highest estimated reward and selects the higher one if these
intervals do not overlap; otherwise, it selects uniformly at random from the
entire convex body. We prove fairness and bound regret by analyzing the rate at
which random exploration shrinks our confidence intervals and relating it to
the frequency of exploitation, a function of $\gap$. We begin by formally
defining $\gap$ below.

\begin{definition}[Gap, adapted from~\citet{DHK08}]
  Given sequence of action sets $\C{} = (\C{1}, \ldots, \C{T})$,
  define $\Omt$ to be the set of extremal points of $\Ct$, i.e.  the
  points in $\Ct$ that cannot be expressed as a proper convex
  combination of other points in $\Ct$, and let
  $\xts= \max_{x \in \Ct} \vecdot{\b}{x}$.  The \emph{gap} of $\Ct$ is
  \[\gap = \min_{1 \leq t \leq T} \left(\inf_{\xt \in \Omt, \xt \neq
      \xts} \vecdot{\b}{\xts - \xt}\right).\]
\end{definition}

$\gap$ is a lower bound on difference in payoff between the
optimal action and any other extremal action in any $\Ct$. When
$\gap > 0$, this implies the existence of a unique optimal action in
each $\Ct$. Our algorithm (implicitly) and our analysis (explicitly)
exploits this quantity: a larger gap enables us to confidently
identify the optimal action more quickly.

We now present the regret and fairness guarantees for $\eeGap$.

\label{thm:gapfair}
\begin{theorem}
	Given sequence of action sets $\C{} = (\C{1}, \ldots, \C{T})$ where
	each $\Ct$ has nonzero Lebesgue measure and is contained in a ball
	of radius $r$ and feedback with $R$-sub-Gaussian noise, $\eeGap$ is fair and
	achieves
	\[
		\regret{T} = O\left(\frac{r^6R^2 \ln(2T/\delta)}
		{\kappa^2 \lambda^2 \gap^2}\right)
	\]
	where $\kappa = 1 - r\sqrt{\frac{2\ln\left(\frac{2dT}{\delta}\right)}
    {T\lambda}}$ and $\lambda = \min_{1 \leq t \leq T} \left[ \lambda_{\min}
    (\Ex{\xt \sim_\uar \Ct}{{\xt}^T\xt})\right]$
\end{theorem}

A full proof of $\eeGap$'s fairness and regret bound, as well as
pseudocode, appears in the supplement. We sketch the proof here:
our proof of fairness proceeds by bounding the influence of noise on
the confidence intervals we construct (via matrix Chernoff bounds) and
proving that, with high probability, $\eeGap$ constructs correct
confidence intervals.  This requires reasoning about the spectrum of
the covariance matrix of each choice set, which is governed by
$\lambda$, a quantity which, informally, measures how quickly we learn
from uniformly random actions.  \footnote{$\lambda$ can be computed
  directly for finite $\Ct$ or approximated by any positive lower
  bound for infinite $\Ct$ and substituted directly into our
  results.}. With correct confidence intervals in hand, fairness
follows almost immediately, and to bound regret we analyze the rate at
which these confidence intervals shrink.

The analysis above implies identical regret and fairness guarantees when each
$\Ct$ is finite. For comparison, the results of Section~\ref{sec:finite}
guarantee $\regret{T} = O(dk\sqrt{T})$. This result, in comparison, enjoys a
regret independent of $k$ which may prove especially useful for cases involving
large $k$.

Finally, our analysis so far has elided any computational efficiency issues
arising from sampling randomly from $\C{}$. We note that it is possible to
circumvent this issue by relaxing our definition of fairness to
\emph{approximate fairness} and obtain similar regret bounds for an efficient
implementation. We achieve this using results from the broad literature on
sampling and estimating volume in convex bodies, as well as recent work on
finding ``2nd best'' extremal solutions to linear programs. Full details appear in
Section~\ref{sec:appx} of the Supplement. 
\section{Instance-dependent Lower Bound for Fair Algorithms}
\label{sec:lower_bound1}
We now present a lower bound instance for which any fair algorithm \emph{must}
suffer gap-dependent regret. More formally, we show that when each choice set is
a square, i.e. $\Ct = [0,1]^2$ for all $t$, for any fair algorithm $\regret{T} =
\tilde{\Omega}(1/\gap)$ with probability at least $1-\delta$. This
also implies the weaker result that no fair algorithm
enjoys an instance-independent sub-linear regret bound $o(T)$ holding
uniformly over all $\b$. We therefore obtain a clear separation between fair
learning and the unconstrained case \cite{DHK08}, and show that an
instance-dependent upper bound like the one in Section \ref{sec:infinite} is
unavoidable. Our arguments establish fundamental constraints on fair learning
with large choice sets and quantify through the $\gap$ parameter how
choice set geometry can affect the performance of fair algorithms. The lower
bound employs a Bayesian argument resembling that in ~\cite{JKMR16} but
with a novel ``chaining'' argument suited to infinite action sets.
We present the result for $d = 2$ for simplicity; the proof technique holds in
any dimension $d \geq 2$.

 \begin{theorem}\label{thm:u1}
   For all $t$ let $\Ct  = [-1, 1]^d$,
   $\b \in [-1,1]^d$,
   and $y_t = \vecdot{x_t}{\b} + \eta_t, $ where
   $\eta_t \sim U[-1,1]$. Let $\A$ be any fair algorithm. Then for
   every gap $\gap$, there is a distribution over instances with gap
   $\Omega(\gap)$ such that any fair algorithm has regret
   $\regret{T} = \tilde{\Omega}(1/\gap)$ with probability $1-\delta$.
   \end{theorem}

   We a sketch of the central ideas in the proof, relegating a full proof to
   the Supplement. We start with the fact that any fair algorithm $\A$ is required to be fair for
   any value $\b$ of the linear parameter. Thus if we draw
   $\b \sim \prior$,
   $\A$ must be round-fair for all $t \geq 1$
   with
   probability at least $1-\delta$, where now the probability includes
   the random draw $\b \sim \prior$. Then Bayes' rule implies that
   the procedure that draws $\b \sim \prior$ and then plays according
   to $\A$ is identical to the procedure which at each step $t$
   re-draws $\b$ from its posterior distribution given the past
   $\prior|_{h_t}$.

Next, given the prior $\prior$, $\A$'s round fairness at step $t$ requires that
(with high probability) if $\A$ plays action $x$ with higher probability than
action $y$, we must have

\begin{equation}\label{eq:better}
	\Pr{\b \sim \prior|h_t}{\vecdot{\b}{x} >
	\vecdot{\b}{y}} > \frac{3}{4}.
\end{equation}

This enables us to reason about the fairness and regret of the algorithm via a
specific analysis of the posterior distribution $\prior|_{h_t}$. We formalize
this argument in Lemmas \ref{lem:posterior} and \ref{lem:posterior-fair}. This
Bayesian trick, first applied in \cite{JKMR16}, is a general technique useful
for proving fairness lower bounds.

We then show that for a choice of prior specific to our choice set $\C{}$, that
two things hold: (i) whenever $\prior|_{h_t} = \prior$, Equation~\ref{eq:better}
forces $\A$ to play uniformly from $\C{}$, and (ii) with high probability $\prior =
\prior|_{h_t}$ until $t > \tilde{\Omega}(1/\epsilon)$, where $\epsilon$ is
a parameter of the prior that acts as a proxy for $\gap$. Playing an action uniformly
from $\C{}$ incurs $\Omega(1)$ regret per round, so these two facts combine to
show that with high probability
$\regret{T} = \tilde{\Omega}(1/\epsilon)$.

Finally we consider $\regret{T}$ conditional on the event that
$\gap(\b) > \delta\cdot\epsilon$, which by our construction of $\prior$
happens with probability $1-\delta$.  Let $\prior_{gap}$ be the conditional
distribution of $\b$ given that $\gap(\b) > \delta\cdot \epsilon$.
Then
 $$\Pr{\b \sim \prior}{\regret{T} \geq \Omega\left(\frac{1}{\epsilon}\right)}\leq
 \Pr{\b  \sim \prior_{gap}}{\regret{T} \geq \Omega\left(\frac{1}{\epsilon}\right)}(1-\delta) + \delta$$
 which implies
$$ \Pr{\b  \sim \prior_{gap}}{\regret{T} \geq \Omega\left(\frac{1}{\epsilon}\right)} \geq
\frac{1-2\delta}{1-\delta}.$$
Note that $\frac{1-2\delta}{1-\delta} \to 1$ as $\delta \to 0$, and so this is a high-probability bound. Since for every $\b$ in the support of $\prior_{gap}$, we have that $\gap(\b) \geq \delta\cdot \epsilon$, we've exhibited a distribution $\prior_{gap}$ such that when $\b \sim \prior_{gap}$, with high probability, $\regret{T} = \tilde{\Omega}(1/\epsilon) = \tilde{\Omega}(1/\gap),$  as desired.

The proof uses the fact that when $\prior = \prior|_{h_t}$, Equation~\ref{eq:better}
forces $\A$ to play uniformly at random. This happens by transitivity: if
Equation~\ref{eq:better} forces $\A$ to play $x$ equiprobably with
$y$ and $y$ equiprobably with $z$, then $x$ must be played equiprobably with
$z$. The fact that any two actions in $\C{}$ can be connected
via such a (finite) transitive chain is illustrated in Figure~\ref{fig:zag}
and formalized in Lemma ~\ref{lem:chain}.

\begin{remark}
We note that this impossibility result only holds for $d \geq 2$. When $d = 1$,
the choice set reduces to $[-1,1]$, and similarly $\b \in [-1,1]$. Thus, the
optimal action is $\mathrm{sign}(\b)$). It takes $O(1/\b^2)$ observations to
determine the sign of $\b$, so a simple fair algorithm may play randomly from
$[-1,1]$ until it has determined $\mathrm{sign}(\b)$, and then play
$\mathrm{sign}(\b)$ for every following round. Because the maximum
per-round regret of any action is $O(\b)$, and because the maximum cumulative
regret obtained by the algorithm is with high probability $O(\b \cdot 1/\b^2) =
O(1/\b)$, the regret of this simple algorithm over $T$ rounds is
$O(\min(\b\cdot T, 1/\b^2))$. Taking the worst case over $\b$, we see that this
quantity is bounded uniformly by $O(\sqrt{T})$, a sublinear parameter
independent regret bound.

\end{remark} 
\section{Zero Gap: Impossibility Result}
\label{sec:lower_bound2}

Section~\ref{sec:infinite} presents an algorithm for which the
sublinear regret bound has dependence $1/\gap^2$ on the instance
gap. Section~\ref{sec:lower_bound1} exhibits an choice set $\C{}$
with a $\tilde \Omega(1/\gap)$ dependence on the gap parameter. We now exhibit
a choice set
$\C{}$ for which $\gap = 0$ for every $\b$, and for which no fair algorithm can obtain non-trivial regret for any value of $\beta$. This precludes even instance-dependent fair regret bounds on this action space, in sharp contrast with the unconstrained bandit setting. 

\begin{theorem}\label{thm:lb-two}
  For all $t$ let $\Ct = S^{1}$, the unit circle, and
  $ \ett \sim \text{Unif}(-1,1)$.  Then for any fair algorithm
  $\A,~\forall \b \in S^{1}, \forall T \geq 1,$
  we have
 $$\Ex{\b}{\regret{T}} = \Omega(T).$$
\end{theorem}

$S^1$ makes fair learning difficult for the following
reasons: since $S^1$ has no extremal points, there is no finite set of
points which for any $\b$ contains the uniquely optimal
action, and for any point in $S^1$, and any finite set of observations, there is another point in $S^1$ for which the algorithm cannot confidently determine relative reward. Since this property holds for \emph{every} point, the fairness constraint transitively requires that the algorithm play every point uniformly at random, at every round.



\newpage
{\footnotesize
\bibliographystyle{plainnat}
\bibliography{nips_bib}}
\newpage
\section{Appendix}\label{sec:appendix}
\subsection{Sub-Gaussian Definition}\label{sec:subgaussian}
Sub-Gaussian random variables have moment generating functions bounded
by the Gaussian moment generating function, and hence can be
controlled via Chernoff bounds.
\begin{definition}
  A random variable $X$ with  $\mu = \E{X}$ is \emph{$R > 0$
  sub-Gaussian} if, for all $t \in \mathbb{R}$, $\E{e^{t(X-\mu)}}\leq e^{Rt^2/2}$.
\end{definition}

\subsection{Proofs from Section~\ref{sec:finite}}
We start with full pseudocode for $\ridgefair_m$.
\newcommand{\pickleq}{\ensuremath{\textsc{pick}_\leq}\xspace}
\vspace{10pt}
\begin{figure}[h]
	\begin{algorithmic}[1]
	\Procedure{$\ridgefair_m$}{$\delta, T, k, \gamma \geq 1 $, ExactBool}
        \For{$t \geq 1, 1 \leq i \leq k$}
        \State Let $\bXt, \bYt =$ design matrix, observed payoffs before round $t$
        \State Let $\Ct$ be the choice set in round $t$
        \State Let ${\Vt} = {\bXt}^{T}\bXt + \gamma I$
        \State Let $\bht = {\Vti}{\bXt}^{T} \bYt$ \Comment{regularized least squares estimator}
        \State Let $\eyi = \vecdot{\bht}{\xti} $ for each $\xti\in\Ct$
     \State Let $\wti = ||\xti||_{\Vti}(\sqrt{2d \log(\frac{1 +
            t/\gamma}{\delta})} + \sqrt{\gamma})$ 
        \State Let $\interval{t}{i} = [\eyi - \wti, \eyi + \wti ]$ \Comment{Conf. int. for  $\eyi$} 
        \If{ExactBool}
        \State \pick$(m, \{(\xti, \interval{t}{i})\})$ 
        \Else \;
        \pickleq$(m, \{(\xti, \interval{t}{i})\})$
        \EndIf  
        \State Update design matrices $\bX{t+1} = \bX{t} :: \Xt, \bY{t+1} = \bY{t} :: \Yt$.
        \EndFor
 	\EndProcedure
	\Procedure{\pick}{$m, (\xtp{1}, \interval{t}{1}), \ldots, (\xtp{k}, \interval{t}{k})$}
        \State Let $M = \Ct$
        \State Let $\Pt = \emptyset$
        \While{$|\Pt| < m$} 
        \State Let $\xtp{\hat{i}} = \argmax_{\xti \in M} \up{t}{i}$ \Comment{Highest UCB not yet selected} 
        \State Let $ \S{t}$ be the set of actions in $\Ct$ chained to $\xtp{\hat{i}}$ \Comment{Highest chain not yet selected} 
        \If{$|\S{t}| \leq m - |\P{t}|$}
        \State $\P{t} = \P{t} \cup \S{t}$ \Comment{Take the chain with probability $1$} 
        \State $M = M \setminus \S{t}$
        \Else
        \State  Let $Q_t $ be $ m - |\P{t}|$ actions chosen UAR from $\S{t}$ 
        \State Let $\P{t} = \P{t}\cup Q_t$ \Comment{fill remaining capacity UAR from the chain} 
        \EndIf
        \EndWhile
        \State   Play $\Pt$
        \EndProcedure
	\Procedure{\pickleq}{$m, (\xtp{1}, \interval{t}{1}), \ldots, (\xtp{k}, \interval{t}{k})$}
           \State Let $\Pt = \{  \textrm{all actions chained to any $\xti\in \Ct$ with $\up{t}{i} > 0$ }\}$ 

\State Let $M = \Ct$
        \State Let $\Pt = \emptyset$
        \While{$|\Pt| < m$ and $\up{t}{\xtp{\hat{i}}}> 0$ for  $\xtp{\hat{i}} = \argmax_{\xti \in M} \up{t}{i}$ }
        \State Let $ \S{t}$ be the set of actions in $\Ct$ chained to $\xtp{\hat{i}}$ \Comment{Highest chain not yet selected} 
        \If{$|\S{t}| \leq m - |\P{t}|$}
        \State $\P{t} = \P{t} \cup \S{t}$ \Comment{Take the chain with probability $1$} 
        \State $M = M \setminus \S{t}$
        \Else
        \State  Let $Q_t $ be $ m - |\P{t}|$ actions chosen UAR from $\S{t}$ 
        \State Let $\P{t} = \P{t}\cup Q_t$ \Comment{fill remaining capacity UAR from the chain} 
\EndIf
\EndWhile
           \State Play $\Pt$
        \EndProcedure
 	\end{algorithmic}
        \caption{$\ridgefair_m$, a fair no-regret algorithm for
          picking $\leq m$ actions whose payoffs are linear.}\label{alg:ridgefair}
\end{figure}

\begin{proof}[Proof of Theorem~\ref{thm:regret-finite}]
  We first claim that confidence intervals are valid: that with
  probability $1-\delta$, for all $t\in [T]$ and all $\xti\in \Ct$,
  $\yti \in \interval{t}{i}$. Assuming this claim, we prove that
  $\ridgefair_m$ is fair.  With probability $1-\delta$, for all rounds
  $t$ and all individuals $\xti$,
  $\yti \in [\eyi - \wti, \eyi+ \wti]$. So, for any pair of
  individuals $\xti, \xtp{j} \in \Ct$, if $\yti > \ytp{j}$, then
  $\eyi + \wti \geq \ey{j} - \w{j}$. So, if $j$ belongs to some chain
  from which arms are selected, either $i$ belongs to a higher chain
  or the same chain as $j$. Every individual belonging to a higher
  chain is played with weakly higher probability to any individual
  belonging to a lower chain, and every two individuals belonging to
  the same chain are played with equal probability, so $i$ is played
  with weakly higher probability than $j$.  Thus, at all rounds and
  for all pairs of individuals, the fairness constraint is satisfied
  by this distribution over $\Pt$, and so $\ridgefair_m$ is fair.

  We now prove the confidence intervals are valid: that with
  probability $1-\delta$, for all $t\in [T]$ and all $\xti\in \Ct$,
  $\yti \in \interval{t}{i}$.  We adopt the notation in~\citet{APS11}:
  let ${\Vt} = {\bXt}^{T}\bXt + \gamma I$, where $\bXt$ is the design
  matrix at time $t$. Let $\bht = {\Vti}{\bXt}^{T}\bYt$ be the
  regularized least squares estimator at time $t$.

  Consider a feature vector $\xti$ at time $t$. For a $d$-dimensional
  vector $z$ and a $d \times d$ positive definite matrix $A$, let
  $\langle z_1, z_2 \rangle_{A}$ denote $z_1^{t}Az_2$. Let $\ettb$ be
  the noise sequence prior to round $t$. Then, we have
  $\bht = {\Vti}{\bXt}^{T}(\bXt\b + \ettb)$. Then some
  matrix algebra in the proof of Theorem 2 of~\citet{APS11} shows
\[ \xti \cdot (\bht - \b) = {\xti}^{T}{\Vti}{\bXt}^{T}\ettb - \gamma {\xti}^{T}{\Vti}\b,\]
which using the above notation gives
  \[ \xti \cdot (\bht - \b) = \langle \xti, {\bXt}^{T}\ettb \rangle_{\Vti} - \gamma \langle \xti, \b \rangle_{\Vti}\]
  Applying Cauchy-Schwarz,
  \begin{align*}
   |\xti \cdot \left(\bht - \b\right)| \leq
  ||\xti||_{{\Vti}}(||{\bXt}^{T}\ettb||_{\Vti}
  + \sqrt{\gamma})
  \end{align*}
  which follows from the fact that
  $ ||\b||_{\Vti} \leq \frac{1}{\sqrt{\gamma}}$ (a
  basic corollary of the Rayleigh quotient, and the fact that by
  assumption $||\b|| \leq 1$.
  
  We now present a result derived from~\citep{APS11} that will help us upper
  bound this quantity.  The upper bound on $||\Xtrans \eta||$ at the bottom of
  page 13 of~\citet{APS11} and the upper bound on $\log(\det(\Vt))$ at the top
  of page 15, combined with our assumption that our noise is $1$-sub-Gaussian,
  implies that
  
	\begin{align*}
		||{\bXt}^{T}\ett||_{\Vti} \leq&\; \sqrt{d \log\left(1 + t/d\gamma\right) -
		2\log{\delta}} \\
		=&\;  \sqrt{d \log\left(1 + t/\gamma\right) + 2 \log\frac{1}{\delta}}\\
		\leq&\; \sqrt{2d \log\left(1 + t/\gamma\right) + 2d \log\frac{1}{\delta}}\\
		=&\; \sqrt{2d \log\left(\frac{1}{\delta}\left(1 + t/d\gamma\right)\right)}\\
		\leq&\; \sqrt{2d \log\left(\frac{1}{\delta}\left(1 + t/\gamma\right)\right)}\\
		=&\; \sqrt{2d \log\left(\frac{1 + t/\gamma}{\delta}\right)}.\\
\end{align*}
  
  Using this result and combining the inequalities we get
  that over all rounds $t \geq 0$ with probability $1-\delta$
  \begin{align}
    | \xti \cdot \left( \bht - \b\right)| \leq
    ||\xti||_{\Vti}\left (\sqrt{2d \log\left(\frac{1 +
    t/\gamma}{\delta}\right)} + \sqrt{\gamma}\right)\label{eq1}
   \end{align}
   and therefore the claim that the confidence intervals are valid
   holds.

\paragraph{Regret bound for $\ridgefair_1$}
We now proceed with upper-bounding the regret of $\ridgefair_1$. With
probability $1-\delta$, the confidence intervals are valid. We will
condition on that event for the analysis of the regret of the
algorithm, since this regret bound will hold with high probability
(namely, with probability $1-\delta$).

We start with a bound that will be useful for analyzing the width of our 
confidence intervals. The top of page 15 of of~\citep{APS11} notes that
$\log \det(\Vt) \leq  d \log (\gamma + t/d)$, and we combine this with the 
fact that $\sum_{t= 1}^{T} ||\xti||_{\Vt}^2 \leq 2 \log \det(\VT) - 2\log
\det(V)$ (proven as part of Lemma 11 in~\citep{APS11}) to get that
  \begin{equation}
     \sum_{t = 1}^{T} ||\xti||_{\Vti}^2 \leq 2d \log\left(1 + \frac{
  T}{d\gamma}\right). \label{eq2}
  \end{equation}

We now have all the tools needed to analyze the algorithm's regret.
First note that the choice of the algorithm is a singleton, i.e. that
$\Pt = \{\iht\}$, for some $\iht \in \S{t}$, the active chained set.
Further, since the confidence intervals are valid, $\ist\in\S{t}$ for
the best action $\ist \in \Ct$. By the definition of $\S{t}$, the
instantaneous regret $\rti$ for any $i\in \S{t}$ is at most
$\rti \leq \sum_{j\in \S{t}} \w{j} $ (as any $i\in\S{t}$ is chained to
some other arm in $\S{t}$). So, we have that

\begin{align*}
R(T) &\leq  \sum_t {\rt{\iht}} \\
& \leq \sum_t \sum_{j\in \S{t}}2 \w{j}   \quad\quad\quad \textrm{Conditioning on w.p. $1-\delta$ valid  confidence intervals}\\
& =  2\sum_t |\S{t}| \cdot \E{\w{{\iht}}}   \quad\quad\quad \textrm{When uniformly selecting $\iht\in \S{t}$; note this holds w.p. $1$ conditioned on valid CI}\\
&  \leq 2k \cdot \sum_t \E{\w{\iht}}   \\
& = 2k \cdot \E{\sum_t \w{\iht}}   \quad\quad\quad \textrm{By linearity of expectation}\\
& = 2k \cdot \E{\sum_t ||\xtp{\iht}||_{\Vti}\left(\sqrt{2d \log\left(\frac{1 +
        t/\gamma}{\delta}\right)}+ \sqrt{\gamma}\right)} \quad\quad\quad \textrm{By definition of $\wti$}\\
& = 2k \cdot \E{\sum_t ||\xtp{\iht}||_{\Vti}\left(\sqrt{2d \log\left(\frac{1 +
        t/\gamma}{\delta}\right)}\right)+ {\sum_t ||\xtp{\iht}||_{\Vti}\left(\sqrt{\gamma}\right)}} \\
& \leq 2k  \E{\sqrt{\sum_{t } ||\xtp{\iht}||_{\Vti}^2} \cdot \left(\sqrt{\sum_{t} 2d \log \left(\frac{1+ t/\gamma}{\delta}\right)} + \sqrt{\sum_{t} \gamma}\right)}  \quad\quad\quad \textrm{By Cauchy-Schwartz}\\
& \leq  2k \sqrt{2 d \log\left(1 +  \frac{T}{d\gamma}\right)}\cdot \left(\sqrt{\sum_{t} 2d \log \left(\frac{1+ t/\gamma}{\delta}\right)} + \sqrt{\sum_{t} \gamma}\right)   \quad\quad\quad \textrm{By Equation~\ref{eq2}}\\
& \leq  2k \sqrt{2 d \log\left(1 +  \frac{T}{d\gamma}\right)}\cdot \left(\sqrt{ 2d T \log \left(\frac{1+ T/\gamma}{\delta}\right)} + \sqrt{T \gamma}\right)  \\
\end{align*}
       or
     $R(T) = O\left(d k
       \sqrt{T}\log\left(\frac{T}{\delta}\right)\right) =
     \tilde{O}(dk\sqrt{T})$ for $\gamma=1$, as desired.

\paragraph{Regret bound for $\ridgefair_m$}

This regret bound relies on a similar analysis to $\ridgefair_1$, with
the following changes. The algorithm now selects $m$ individuals, not
all from the top chain, but instead from several chains. For each of
the top $m$ choices $x\in \Pst$, we relate reward of that choice to
the reward of one of our choices in the following way.  For each $d$,
consider the $d$th top chain. We claim that if the $d$th top chain
contains $n_d$ of the top $m$ choices, our algorithm selects $n_d$
individuals from the $d$th top chain. We prove this claim by
induction. First, however, we notice that every individual in the
$d$th top chain has strictly higher reward than every individual in
any lower chain. For the first top chain, $\Pt$ contains either the
entire chain or $m$ from the top chain. As every individuals in the
top chain has strictly higher reward than any other individuals, in
the former case, every individual in the first top chain belongs to
$\Pst$ ; in the latter case, $\Pst$ is entirely contained in the top
chain. Thus, $\Pt$ and $\Pst$ contain either all individuals in the
top chain or exactly $m$ of them.  Then, assuming the claim for the
first $d-1$ top chains, both $\P{t}$ and $\Pst$ have the same
``capacity'' for individuals in the $d$th top chain (and therefore
either take all of the $d$th top chain or fewer but the same number
from it). This proves the claim.

Then, we relate the reward of an $i\in \P{t}$ with some action in
$\Pst$ belonging to same chain. Following the previous claim, we can
form a matching between $\Pst$ and $\P{t}$ for which all matches
belong to the same chains. Then, the analysis of $\ridgefair_1$ bounds
the difference between the reward of any individual in the $d$th top
chain to any other individual in the $d$th top chain. Summing up over
all $m$ choices, the total regret for all of $\Pt$ is at most $m$
times the loss suffered in \LB.
\paragraph{Regret bound for $\ridgefair_{\leq k}$}
The regret bound for this case reduces to lower-bounding the amount of
reward incurred by playing arms with negative reward. Any individual
selected by $\ridgefair_{\leq k}$ is within the sum of the widths of
the confidence intervals in its chain, one of which has UCB which is
positive. So, the reward of any action chosen is at least
$-\sum_{i\in \S{t}^d} \w{t}{i}$ for $\S{t}^d$ the $d$th top chain, or
at most the sum of all $k$ interval widths. Thus, summing up over all
individuals selected, one gets regret which is at most $k$ times worse
than that for $\ridgefair_1$.
\end{proof}

\subsection{Proofs from Section~\ref{sec:infinite}}
We begin with the full pseudocode for~\eeGap.

\vspace{10pt}
\begin{figure}[h]
	\begin{algorithmic}[1]
	\Procedure{$\eeGap$}{$\delta, \C, \lambda$}
        \For{$t \geq 1$}
	        \If{$2r\ln(2dt/\delta)/\lambda \geq t$}
	        	\State Play $\xht \sim_\uar \Ct$
	        	\State Update design matrices $\bX{t+1}, \bY{t+1}$
	        \Else
	        	\State Let $\delta = \min(\delta, 1/t^{1+c})$
                \State Let $\bht= (\Xttrans\bXt)^{-1}\Xttrans\bYt$ 
                \Comment{Least squares estimator}
                \State Let $\kappa = 1 - r\sqrt{2\ln(2dt/\delta)/t\lambda}$
                \State Let $\wt = \frac{r^2 \cdot R \cdot 2\sqrt{\ln(2t\delta)}}
                {k\lambda \sqrt{t}}$ \Comment{Confidence interval width}
                \State Let $(x_1,x_2) = \toptwo(\Ct, \bht)$
	            \Comment{Find two ext. pts. maximizing $\langle x,
                \bht\rangle$}
	            \State Let $U_1 = [\vecdot{\bht}{x_1} - \wt,
	            \vecdot{\bht}{x_{1}} + \wt]$
	            \State Let $U_2 = [\vecdot{\bht}{x_2} - \wt,
	            \vecdot{\bht}{x_{2} } + \wt]$
	            \If {$U_1 \cap U_2 = \emptyset$}
	                \State Let FoundMax = $\{x\}$
	               	\State Play $\xht = x$
	               	\Comment{Play $\xht$ once confidence intervals separate}
	            \Else
	               	\State Play $\xht \sim_\uar \Ct$
	               	\State Update design matrices $\bX{t+1},\bY{t+1}$
	            \EndIf
	        \EndIf
	    \EndFor
        \EndProcedure
 	\end{algorithmic}
        \caption{$\eeGap$, a fair no-regret algorithm for infinite, changing 
        action sets.}\label{fig:eegap}
\end{figure}

We start our proof of Theorem~\ref{thm:gapfair} with a lemma bounding the
contribution of noise to our confidence intervals.
\begin{lemma}\label{lem:noise}
	Let $\et{1}, \ldots, \et{T}$ be $T$ i.i.d draws of $R$-sub-Gaussian noise. Then
	\[
		\Pr{}{\left|\sum_{i=1}^T \et{i}\right| \geq
		R\sqrt{2T\ln(2T/\delta)}} \leq \delta/2T.
	\]
\end{lemma}
\begin{proof}[Proof of Lemma~\ref{lem:noise}]
	A Hoeffding bound, in the general case for unbounded variables, implies that 
	\[
		\Pr{}{\left|\sum_{i=1}^T \eta_i \right| \geq c} \leq 2\exp(-c^2/2R^2)
	\]
	so taking $c = R\sqrt{2T \ln\left(2T/\delta\right)}$ yields the desired
	result.
\end{proof}

Next, since the regret bound we will prove depends on
$\lambda = \min_{1 \leq t \leq T} \left[ \lambda_{\min}(\Ex{\xt \sim_\uar \Ct}
{{\xt}^T\xt})\right]$, the minimum smallest eigenvalue of the expected outer
product of a vector $\xt$ drawn uniformly at random from each $\Ct$
we will need $\lambda > 0$ in order for this
bound to make sense. We prove this in another lemma.

\begin{lemma}\label{lem:positive}
	Given sequence of action sets $\C{} = (\C{1}, \ldots, \C{T})$ where each $\Ct$
  	has nonzero Lebesgue measure and is contained in a ball of radius $r$,
  	$\lambda = \min_{1 \leq t \leq T} \left[ \lambda_{\min}(\Ex{\xt
  	\sim_\uar \Ct}{{\xt}^T\xt})\right] > 0.$
\end{lemma}
\begin{proof}[Proof of Lemma~\ref{lem:positive}]
  It suffices to prove that $\lambda_{\min}(\Ex{\xt \sim_\uar \\Ct}{\xt^T\xt})
  > 0$ for each $1 \leq t \leq T$.
  $x^Tx$ is positive semidefinite, so it is immediate that $\lambda \geq 0$.
  Assume $\lambda = 0$. Then there exists nonzero $z \in \mathbb{R}^d$ such
  that $z\Ex{x \sim_{\uar} \C{}}{x^Tx}z^T = 0$, so by linearity of
  expectation
  $$\Ex{x \sim_{\uar} \C{}}{||xz^T||^2} = 0.$$
  However, $||xz^T||^2$ is a non-negative random-variable with
  expectation $0$ and must therefore be $0$ with probability $1$. It follows 
  that $x \in z^{\perp}$, so
  $$\Pr{x \sim_{\uar} \C{}}{x \in z^{\perp}} = 1,$$ 
  $z^{\perp}$ is a $d-1$ dimensional subspace of $\mathbb{R}^{d}$, and thus has
  measure $0$. We can decompose $\C{} = (\C{} \cap z^{\perp}) \bigcup (\C{} \cap
  \left(z^{\perp}\right)^{c})$, and since
  $\Pr{x \sim_{\uar} \C{}}{x \in z^{\perp}} = 1,$ this forces
  $\Pr{x \sim_{\uar} \C{}}{x \in (\C{} \cap {z^{\perp}}^{c})} = 0$. By
  definition of the uniform distribution
  $$\Pr{x \sim_{\uar} \C{}}{x \in (\C{} \cap {z^{\perp}}^{c})} =
  \frac{\mu(\C{} \cap {z^{\perp}}^{c})}{\mu(D)} \implies \mu(\C{} \cap
  {z^{\perp}}^{c}) = 0.$$
  But
  $\mu(D) = \mu(\C{} \cap z^{\perp}) + \mu(\C{} \cap {z^{\perp}}^{c}) =
  \mu(\C{} \cap z^{\perp}) + 0 \leq \mu(z^{\perp}) = 0$,
  where the second to last line follows since
  $\C{} \cap z^{\perp} \subset z^{\perp}$.  This contradicts our
  assumption that $\mu(\C{}) > 0$, so $\lambda > 0$.
\end{proof}

Finally, since \eeGap relies on constructed confidence intervals to
guide its choice of actions, its correctness (both in terms of its
regret guarantee and its fairness) relies on the correctness of those
confidence intervals, stated in the following lemma. Its proof relies
on a natural argument using matrix Chernoff bounds to bound the
contribution of noise to \eeGap's estimation of $\bht$ and,
consequently, the accuracy of its confidence intervals.

\begin{lemma}\label{lem:CIs}
  Given sequence of action sets $\C = (\C{1}, \ldots, \C{T})$ where
  each $\Ct$ has nonzero Lebesgue measure and is contained in a ball
  of radius $r$, with probability at least $1-\delta$, in every round
  $t$ every confidence interval
  $[\vecdot{\bht}{x} - \wt, \vecdot{\bht}{x} + \wt]$ constructed by
  \eeGap contains its true mean $\vecdot{\b}{x}$.
\end{lemma}
\begin{proof}[Proof of Lemma~\ref{lem:CIs}]
  Note first that $\eeGap$ has two kinds of rounds: in round $t$, it either
  plays uniformly at random from $\\Ct$ or deterministically plays $\hat 
  \xt^*$, its estimate of the optimal extremal point in $\\Ct$. In any round 
  $t$ with uniform random play $\eeGap$ immediately cannot violate fairness, as
  $\pi_t(x) = 1/\mu(\\Ct)$ for all $x \in \Ct$. As a result, to prove fairness
  it suffices to show that for any $t$-step execution of $\eeGap$,
  \[
  	\Pr{\C{1}, \ldots, \\Ct}{\text{deterministically play } \hat x_i^* \neq
	x_i^* \text{ in any round } i} \leq \delta
  \]
        $x_i^*$ is the true optimal point in $\C{i}$, and $t > 4dr^4/\delta
        \lambda^2$ (since for smaller $i$ \eeGap just plays uniformly at
        random).
	
	In round $t+1$ after observing $x_1 \sim_\uar \C{1}, \ldots, \xt \sim_\uar
	\\Ct$, for every $x \in \Omega$ we have
	\begin{align*}
		|\vecdot{x}{\hat \b - \b}| =&\; |\vecdot{x}{(X^TX)^{-1}X^T(X\b +
		\eta) - \b}| \\
		=&\; |x^T\b + x^T(X^TX)^{-1}X^T\eta - x^T\b| \\
		=&\; |x^T(X^TX)^{-1}X^T\eta|
	\end{align*}
	where $X \in \mathbb{R}^{t \times d}$ is the design matrix of
        $x_1, \ldots, x_i$ and $\eta \in \mathbb{R}^t$ is its noise
        vector. We can then decompose $X^T \eta$ by round as
	\begin{align*}
		|x^T(X^TX)^{-1}X^T\eta| =&\; \left| x^T(X^TX)^{-1} \sum_{i=1}^t
		x_i\eta_i\right| \\
		=&\; \left| \sum_{i=1}^t x^T(X^TX)^{-1}x_i \eta_i\right| \\
		\leq&\; \sum_{i=1}^t \left[ ||x^T(X^TX)^{-1}x_i|| \cdot
		\left|\eta_i\right| \right] \\
		\leq&\; \sum_{i=1}^t \sqrt{xx^T \cdot x_ix_i^T} \cdot
		\lambda_{\max}((X^TX)^{-1})	\cdot \left|\eta_i \right| \\
		\leq&\; r^2 \cdot \lambda_{\max}((X^TX)^{-1}) \cdot
		\left| \sum_{i=1}^t \eta_i \right| \\
		=&\; \frac{r^2}{\lambda_{\min}(X^TX)} \cdot \left|\sum_{i=1}^t \eta_i \right|
	\end{align*}
	where the second inequality follows from the fact that
	\[
		||(X^TX)^{-1}|| = \sqrt{\lambda_{\max} ([X^TX]^{-1}[(X^TX)^{-1}]^T)}
		= \sqrt{\lambda_{\max}([(X^TX)^{-1}]^2)} = \lambda_{\max}((X^TX)^{-1}),
	\]
	the third inequality follows
	from the assumed bound on each $\C{i}$, and the final equality follows from 
	$\lambda_{\max}(A^{-1}) = \frac{1}{\lambda_{\min}(A)}$.
	To upper bound this quantity, we now lower bound $\lambda_{\min}(X^TX)$.
	
	To do so, we first note that for any $1 \leq i \leq t$ and any $x \in \C{i}$ we
	have $\lambda_{\max}(x^Tx) \leq r^2$ by the Gershgorin circle theorem, which
	states that a square matrix has maximum eigenvalue bounded by its largest
	absolute row or column sum. Next, by linearity of expectation
	\[
        	\lambda_{\min}(\Ex{x_1 \sim_\uar \C{1}, \ldots, \xt \sim_\uar	\\Ct}
        	{X^TX}) = \lambda_{\min}\left(\sum_{i=1}^t\Ex{x_i \sim_\uar \C{i}}{x_i^Tx_i}\right)
        \geq  t\lambda
    \]
    for $\lambda = \min_{1 \leq i \leq t} \left[ \lambda_{\min}(\Ex{x_i
    \sim_\uar \C{i}}{x_i^Tx_i})\right]$.Taking this together with a matrix
    Chernoff bound (see e.g.~\cite{T15}) yields
	\[
		\Pr{}{\lambda_{\min}(X^TX) \leq \kappa t\lambda} \leq
		de^{(-(1-\kappa)^2t\lambda/2r^2)}
	\]
	for any $\kappa \in [0,1)$. Setting
	\[
		\kappa = 1 - \sqrt{\frac{2r^2\ln\left(\frac{2dt}{\delta}\right)}{t\lambda}}
	\]
	this implies
	\[
		\Pr{}{\lambda_{\min}(X^TX) \leq \kappa t \lambda} < \frac{\delta}{2t}
	\]
	where $\kappa \in [0,1)$ since $t > 2r^2\ln(2dt/\delta)/\lambda$. Combining this
	with Lemma~\ref{lem:noise} and a union bound, we get that with probability
	$\geq 1 - \delta/t$
	\begin{align*}
		\frac{r^2}{\lambda_{\min}(X^TX)} \cdot \left|\sum_{i=1}^t \eta_i \right|
		\leq&\; \frac{r^2}{\kappa t\lambda} \cdot R\sqrt{2t\ln(2t/\delta)} \\
		=&\; \frac{r^2 R \sqrt{2\ln(2t/\delta)}}{\kappa \lambda \sqrt{t}}.
	\end{align*}
	Taking a union bound over $t$ rounds, it follows that with
        probability at least $ 1-\delta$ through $t$ rounds every
        constructed confidence interval around
        $\vecdot{\hat \b}{x}$ contains $\vecdot{\b}{x}$. Since
        $\eeGap$ only plays $\hat x^*$ deterministically when the
        confidence intervals around $\hat x^*$ and other extremal
        points do not overlap, this means that with probability at
        least $1-\delta$ $\eeGap$ correctly identifies $x^*$. \eeGap is
        therefore fair.
\end{proof} 

Taken together, these lemmas let us prove Theorem~\ref{thm:gapfair}.

\begin{proof}{Proof of Theorem~\ref{thm:gapfair}}
	We begin by proving fairness. By Lemma~\ref{lem:CIs}, with probability at
	least $1-\delta$ every confidence	interval constructed by \eeGap contains its
	true mean. Conditioning on this correctness of confidence intervals, since
	\eeGap only chooses an	action $x_1$ non-uniformly when
	$U_1 \cap U_2 = \emptyset$, it follows that
	any action chosen non-uniformly by \eeGap is optimal. Thus, with probability
	at	least $1-\delta$ \eeGap never chooses a suboptimal action $x$ with higher
	mixture density $\ppt(x)$ than a superior action $x'$, and \eeGap is fair.
  
  While $\eeGap$ plays at random from $\C{}$ (for some number of rounds
  at least $4dr^4/\delta \lambda^2$), it incurs at most $2r$ regret
  per round. The algorithm incurs $0$ regret once the confidence
  intervals around the top two extremal points no longer intersect. A
  sufficient condition is therefore
	\[
		\frac{r^2 \cdot R \cdot \sqrt{2\ln(2T/\delta)}}{\kappa \lambda \sqrt{T}} <
		\frac{\gap}{2}
	\]
	which we rearrange into
	\[
		\frac{8r^4R^2 \ln(2T/\delta)}{\kappa^2 \lambda^2 \gap^2} < T.
	\]
        After this many rounds, with probability $\geq 1-\delta$,
        $\eeGap$ identifies the optimal arm in \emph{every} $\Ct$ and incurs
        no further regret.

Thus, the regret in total is at most
\[ \sum_{t =1}^L 2 r^2 + \delta T \leq \frac{16r^6R^2 \ln(2T/\delta)}
{\kappa^2 \lambda^2 \gap^2} + \delta T\]
where
$L = \frac{8r^4R^2 \ln(2t/\delta)}{\kappa^2 \lambda^2 \gap^2}$
and $\delta \leq 1/(T^{1+c})$ then implies the claim.
\end{proof}

\subsection{Efficient Approximate Version of Section~\ref{sec:infinite}}
\label{sec:appx}

In this section we describe an efficient implementation of \eeGap using 
approximate fairness.

Recall that \eeGap requires some method of sampling uniformly at
random from a given convex body $\Ct$, a problem that has attracted
extensive attention over the past few decades (see~\cite{V05} for a
survey of results). For our purposes, the primary contribution of this
literature is that one cannot do better than \emph{approximately}
uniform random sampling from a convex set $\Ct$ under polynomial time
constraints.

Since our current definition of fairness fails without a perfectly
uniform distribution over actions, efficiency necessitates a
relaxation of our definition to approximate fairness for infinite
action spaces.  Intuitively, approximate fairness will require that an
algorithm (with high probability) always uses a distribution that is
at least ``almost" fair.

\begin{definition}[$\epsilon$-Approximate Fairness]
	Given sequence of action sets $\C{} = (\C{1}, \ldots, \C{T})$, we say that
	algorithm $\A$ is \emph{$\epsilon$-approximately fair} if, for any inputs
	$\delta \in (0,1],	\epsilon > 0$ and for all $\b$, with probability at
	least $1-\delta$ at every round $t$ there exists a fair distribution ${\ppt}^f$
	such that $$||{\ppt} - {\ppt}^f|| < \epsilon$$
	where ${\ppt}$ is $\A$'s choice distribution over $\Ct$ in round $t$ and
	$\|\cdot\|$ 	denotes total variation distance.
\end{definition}

We call this $\epsilon$-approximate fairness to highlight that a single 
$\epsilon$ is input to the algorithm $\A$ in question, but will often shorthand
this as \emph{approximate fairness}.

Below we provide an approximately fair algorithm that, subject to additional
assumptions on choice set structure, obtains similar regret guarantees as 
$\eeGap$ efficiently. We modify $\eeGap$ as follows: first, we replace each
call to a random sample with a hit-and-run random walk scheme~\citet{LV06}
to efficiently sample approximately uniformly at random. 

We use the following lemma from~\citet{LV06} to upper-bound the mixing
time hit-and-run requires to approach a near-uniform distribution in
its walk over $\Ct$.

\begin{lemma}\label{lem:sample}[Corollary 1.2 in~\cite{LV06}]
	Let $S$ be a convex set that contains a ball of radius $r'$ and is
	contained in a ball of radius $r$. Then, starting from a point $x \in S$ at a
	distance $\alpha$ from the boundary, after
	$$c > 10^{11}d^3\left(\frac{r}{r'}\right)^2\ln\left(\frac{r}{\alpha \epsilon}
	\right)$$
	steps of a hit-and-run random walk the random walk induces a probability
	distribution $P$ over points in $S$ such that $P$ is $\epsilon$-close to
	uniform in total variation distance.
\end{lemma}

Next, we show that, with an additional assumption on the structure of $\Ct$,
\eeGap 's subroutine \toptwo can be implemented efficiently via the following
known result~\citep{LDK16}.

\begin{lemma}[~\citet{LDK16}]\label{lem:toptwo}
  Let $\Ct$ be defined by $m$ intersecting half-planes. Then there exists an
  algorithm running in time polynomial in $m$ and $d$ which
  computes the two vertices which maximize $\bht$ over $\Ct$.
\end{lemma}

This algorithm enables us to compute $\toptwo(\Ct, \bht)$
efficiently.

The following lemma guarantees that the distributions
over histories generated by \eeGap and \aeeGap~are ``close" during exploration.

\begin{lemma}\label{lem:exact_appx}
  Let $\C{} = (\C{1}, \ldots, \C{T})$ be a sequence of action sets where each
  action set, in addition to satisfying the assumptions of
  Theorem~\ref{thm:gapfair}, is an intersection of polynomially many
  halfspaces and contains a ball of radius $r'$. Then through $t$ rounds of
  exploration
  \[
  	||P_{\pp{1}, \ldots, {\ppt} \sim \eeGap} - P_{\pp{1}, \ldots, {\ppt} \sim
  	\aeeGap(\epsilon/t)}||_{TV}
  	< \epsilon
  \]
  where each $P$ represents distributions over possible exploration histories
  generated by \eeGap and \\ \aeeGap$(\epsilon/t)$ respectively.
\end{lemma}
\begin{proof}[Proof of Lemma~\ref{lem:exact_appx}]
  By construction, during exploration each $\pp{i}$ output by
  \aeeGap$(\epsilon)$ has a distribution within $\epsilon/t$ of a
  uniform distribution in total variation distance. Since these
  samples are independent, each distribution over
  $\pp{1}, \ldots, {\ppt}$ forms a product distribution, and the
  additivity of total variation distance over product distributions
  implies the claim.
\end{proof}

Combining the results above lets us prove that \aeeGap~is fair, efficient, and
obtains a similar regret bound as \eeGap.

\begin{theorem}\label{thm:appx}
  Consider an action set $\C{}$ that, in addition to satisfying the assumptions
  of Theorem~\ref{thm:gapfair}, is an intersection of polynomially many
  halfspaces and contains a ball of radius $r'$. Then through $T$ steps given
  inputs $\delta' = \delta/2$ and $\epsilon' = \min(\epsilon/T,\delta/2T^2)$,
  \aeeGap$(\epsilon')$ is efficient, $\epsilon$-approximately fair, and
  achieves regret
  \[
	\regret{T} = O\left(\frac{r^6R^2 \ln(4T/\delta)}
{\kappa^2 \lambda^2 \gap^2}\right)
	\]
	where $\kappa = 1 - r\sqrt{\frac{2\ln\left(\frac{2dT}{\delta}\right)}
    {T\lambda}}$.
\end{theorem}
\begin{proof}{Proof of Theorem~\ref{thm:appx}}
  In each round $t$, \eeGap performs (at most) three
  computation-intensive operations. First, it computes a least squares
  estimator $\bht$, which may be maintained online and
  updated in poly$(d)$ time. Next, it calls subroutine
  \toptwo$(\Ct, \bht)$ to compute $(x_1,x_2)$ in $\poly(d,m)$
  time via Lemma~\ref{lem:toptwo}. Finally, it may choose an action
  (approximately) uniformly at random from $\Ct$, which also takes
  polynomial time via Lemma~\ref{lem:sample}. It follows that each
  round $t$ of \eeGap takes polynomial time, so \eeGap is efficient.
  
  To prove that \aeeGap~is approximately fair, as in the exact case we analyze
  \aeeGap's split between exploration and exploitation. In exploration rounds,
  by Lemma~\ref{lem:sample} we know that each random sample is
  $\epsilon/T$-close to a true uniform distribution and therefore satisfies
  $\epsilon$-approximate fairness immediately.
  
  We now bound the probability of violating fairness during
  exploitation. This can only happen if in some round $t$
  \aeeGap~misidentifies the optimal extremal point $\xts$ to exploit
  and instead deterministically plays $\xths \neq \xts$. Since
  \aeeGap~only uses exploration rounds to construct its design matrix,
  the identified $\xths$ is a deterministic function of the
  $k \leq t-1$ exploration rounds $h_1, \ldots, h_k$ seen before round
  $t$.  Lemma~\ref{lem:exact_appx} implies that \eeGap and
  \aeeGap~have distributions over $h_1, \ldots, h_t$ within $\epsilon$
  of each other. We then combine two facts. First, here \eeGap~has at
  most $\delta/2$ probability of constructing incorrect confidence
  intervals assuming perfect uniform random sampling.  Second,
  $\aeeGap$ has probability at most $\epsilon \leq \delta/2T^2$ of
  identifying a $\xths$ different from that of $\eeGap$ by the above
  argument. A union bound then implies that $\aeeGap$ has probability
  at most $\delta/2$ of identifying a different $\xths$ than
  $\eeGap$. Combining the probability of $\eeGap$ failing and
  $\aeeGap$ failing to approximate $\eeGap$, we get that \aeeGap~has
  probability at most $\delta/2 + \delta/2 = \delta$ of misidentifying
  $x^*$. Thus \aeeGap~is $\epsilon$-approximately fair.
  
  To analyze \aeeGap's regret, note that in the case where
  \aeeGap~correctly identifies $\xts$, \aeeGap's use of $\delta/2$
  rather than $\delta$ adds a factor of 2 inside the log in the
  original regret statement of \eeGap. Next, \aeeGap~incorrectly
  identifies $\xts$ with probability at most $\delta$ by the logic
  above, so taking $\delta \leq 1/T^{1+c}$ as in the proof of Theorem
  ~\ref{thm:gapfair} implies the claim.
\end{proof} 

\subsection{Proofs from Section~\ref{sec:lower_bound1}}

\begin{proof}[Proof of Theorem~\ref{thm:u1}]
Let $\A$ be a fair algorithm. For any input $\delta$, $\A$ is round-fair for all
$t\geq 1$ with probability $1-\delta$. Since this holds for any
$\b$ with probability at least $1-\delta$, then it necessarily holds
with probability at least $1-\delta$ over any prior $\prior$
on $\b$
with support contained in the unit rectangle. Our first lemma gives an
alternative way to view the framework which draws $\b \sim \prior$ and
then plays according to $\A$.

Let $\xt$ denote the action chosen by $\A$ at time step $t$, and let $\yt$
denote the observed reward. Let the joint distribution of $((x_1,y_1), \ldots
(\xt, \yt), \b)$ be denoted by $W_t$. Lemma~\ref{lem:posterior} is similar
in content to Lemma $4$ in \cite{JKMR16}; its proof follows from Bayes' Rule.

\begin{lemma}\label{lem:posterior}
  Let $\b'$ at time $t$ be drawn from $\prior|h_t$, its posterior
  distribution given the observed sequence of choices and rewards
  $h_t = ((x_1, y_1), \ldots (x_{t-1}, y_{t-1})) \in (\C{} \times \R)^{t-1}$.
  Then let $W_t'$ be the joint distribution of
  $(h_t, (\xt, \yt), \b')$. $W_t$ and $W'_t$ are identical
  distributions.
\end{lemma}

Lemma ~\ref{lem:posterior} states that whether the instance draws $\b \sim
\prior$ once and then plays according to $\A$, or re-draws $\b$ from its
posterior at each time-step, the joint distribution on instances and
observations is unchanged at each step.  We can thus assume without loss of
generality that, given a prior $\prior$, at each time step $t$ we redraw $\b
\sim \prior|_{h_t}$. Taking this posterior viewpoint, we have the following lemma.

\begin{lemma} \label{lem:posterior-fair}
Given a fixed prior $\prior$, let $\A$ be fair and let $\b \sim \prior$.
Let $\pi_t$ be the distribution on actions of $\A$ at time $t$, and let $f_t$
be the pdf of $\pi_t$. Then with probability at least $1-4\delta$, at
each time $t$, if $\Pr{\prior|h_t}{\vecdot{\b}{y} > \vecdot{\b}{x}} >
\frac{1}{4}$,
then $f_t(y) \geq f_t(x)$.
\end{lemma}

This means that with probability at least $1-4\delta$, whenever the posterior
distribution at time $t$ tells us that point $y$ has a higher reward than point
$x$ with probability at least $\frac{1}{4}$ over the posterior distribution of
$\b$, we must play $y$ with at least the same probability as $x$. 

We will use this lemma, in combination with results about a specific posterior,
to constrain the possible actions any fair algorithm can take.

We now introduce the specific prior $\prior$. Let $\b$ have prior
distribution $\prior \sim \{1\}\; \times \; U[-\epsilon, -\epsilon]$. We first
analyze the posterior distribution of $\b$. We then show that with
probability at least $1-4\delta$, until the posterior distribution differs
from the prior, Lemma \ref{lem:posterior-fair} forces $\A$ to play uniformly
from $\C{t}$.

Suppose that we have observed $(x_1,y_1) \ldots (x_{t-1}, y_{t-1})$.
Since the prior in the second coordinate is $U[-\epsilon, \epsilon]$,
and the noise $\eta_i$ is also uniform, the posterior in the second
coordinate is uniform over all $\b_2$ consistent with the
observed data in the following sense: since the noise $\eta_{t'}$ is bounded, each
  pair $(\xtpr, \ytpr)$
  gives a bound on $\b_2$. Combining $\ytpr = \xtpri{1} + \b_2\xtpri{2}
  + \etapr$
  and $\etapr \in [-1,1]$ we get
\begin{equation}\label{eq:theta_constraint}
\b_2 \in [l_{t'}, u_{t'}] = \left[\min\left(\frac{\ytpr - \xtpri{1} - 1}{\xtpri{2}},
\frac{\ytpr - \xtpri{1} + 1}{\xtpri{2}}\right), \max\left(\frac{\ytpr - \xtpri{1} -1}
{\xtpri{2}}, \frac{\ytpr - \xtpri{1} + 1}{\xtpri{2}}\right)\right].
\end{equation}

Since by the prior we know $\b_2 \in [-\epsilon, \epsilon],$ we 
say that $\b_2$ is consistent with $h_t$ if $\b_2 \in [-\epsilon,
\epsilon]$ and $\b_2 \in [\sup_{t'}l_{t'}, \inf_{t'}u_{t'}]$.
This is the content of the following lemma.

\begin{lemma}\label{lem:consistent}
  Let
  $\ytpr = \vecdot{\b}{\xtpri} + \etapr, \etapr \sim U[-1,1],$ and $\b
  \sim \{1\} \times U[-\epsilon, \epsilon]$.
  Then $\prior(\b_2|h_t)$ is uniform on the set of $\b_2$
  \textit{consistent} with $h_t$.
\end{lemma}

We now define and analyze $S$, the number of rounds required before the
posterior distribution of $\beta_2$ becomes non-uniform. 
Each $(\xtpr, \ytpr)$ gives the constraint on $\b$ given in Equation
~\ref{eq:theta_constraint}. This only changes the posterior from the prior if
$l_{t'} > -\epsilon$ or $u_{t'} < \epsilon$. Assume first that $\xtpri{2} > 0$ (by
symmetry, a similar argument holds for $\xtpri{2} < 0$). 
Then $u_{t'} = \dfrac{\ytpr - \xtpri{1} + 1}{\xtpri{2}}$ and we can calculate
\begin{align*}
   \Pr{}{\frac{\ytpr + 1 - \xtpri{1}}{\xtpri{2}} < \epsilon} =&\;
   \Pr{}{\frac{\etapr + \xtpri{2}\b_2 + 1}{\xtpri{2}} < \epsilon} \\
   =&\; \Pr{}{\etapr + 1 < \xtpri{2}(\epsilon-\b_2)} \\
   \leq&\; \Pr{}{\etapr + 1 < 2\epsilon} = \epsilon
\end{align*}
 where the last equality follows from the fact that
 $\ett + 1 \sim U[0,2]$. The probability that the lower bound is
 greater than $-\epsilon$ is similarly

 \[ \Pr{}{\frac{\ytpr - 1 - \xtpri{1}}{\xtpri{2}} > -\epsilon} = 
 \Pr{}{\etapr > 1 + \xtpri{2}(-\epsilon - \b_2)} \leq \epsilon.\]

 Thus the probability that any pair $(\xtpr, \ytpr)$ alters the posterior
 distribution of $\b_2$ from $U[-\epsilon, \epsilon]$ is at most
 $2\epsilon$. \footnote{Note that this bound holds regardless of the particular
 choice of $\xtpr$, which is why the probabilities above are over the draw of
 the rewards $\ytpr$, conditional on the chosen $\xtpr$.} 
 It follows that $\mathbb{P}(S \geq t') \geq (1-2\epsilon)^{t'}$,
 and that the
 posterior coincides with the prior $\prior$ for $\Omega(1/\epsilon)$ steps in
 expectation.

 Now assume that after $t-1$ steps the posterior distribution is equal
 to $\prior$: we will argue that any non-uniform distribution violates
 round fairness in round $t$ with probability at least
 $\frac{3}{4}$. Call two points $a, b \in \Ct = [ -1,1]^2$ \emph{vertically equivalent}
 if $a_1 = b_1$, i.e. they agree in their first coordinate. Consider
 some pair of points $a = (x_1, x_2), b = (x_1, x_3)\in \Ct$ which are
 vertically equivalent with $x_2 > x_3$. Suppose $\A$ plays $a$ with higher
 probability than $b$. If $\b_2 < 0$, then
 $ \vecdot{\b}{a} < \vecdot{\b}{b}$, and
 $\Pr{}{\b_2 < 0} = 1/2 > \frac{1}{4}$. Thus $\A$ violates round-fairness in
 round $t$ with probability more than $\frac{1}{4}$. Similarly,
 if $\b_2 > 0$, then $ \vecdot{\b}{a} > \vecdot{\b}{b}$,
 and $\Pr{}{\b_2 >0} = 1/2 > \frac{1}{4}$, so if $\A$ plays $b$ with
 higher probability than $a$ then $\A$ again violates round-fairness in round
 $t$ with probability strictly larger than $\frac{1}{4}$. Thus,
 any two vertically equivalent points must be played with equal
 probability.

 Next, consider any point $b \in \Ct$ of the form $\left(x_1 - \alpha, x_2 +
 \frac{2\alpha}{\epsilon}\right)$ for some $\alpha \in \mathbb{R}$. Call any two points of
 this form, for fixed $(x_1,x_2)$ and variable $\alpha \in \mathbb{R}$, 
 \emph{diagonally equivalent}. Let $a = (x_1,x_2)$. If $\b_2 > \epsilon/2$, then
 $\vecdot{\b}{b} \geq x_1 - \alpha + x_2\b_2 + \alpha = x_1 + x_2\b_2 =
 \vecdot{\b}{a}$. Since $\b_2 > \epsilon/2$ with probability $\frac{1}{4}$,
 point $b$ must be played with probability at least that of $a$ to satisfy
 round-fairness in round $t$ with probability greater than $\frac{3}{4}$.
 Symmetrically, when $\b_2< \frac{-\epsilon}{2}$, which happens with
 probability $\frac{1}{4} $, $a$ must have at least as much probability of
 being played as $b$. Thus, any two diagonally equivalent points must also be
 played with equal probability.

Given a point $x \in \Ct$, let $H_x$ denote the transitive closure
under vertical and diagonal equivalence of the point
$x$. Since points that are equivalent must be played with equal
probability, by the transitive property all points in $H_x$ must be
played with equal probability by $\A$. We now show that when $x$ is
a corner of $\Ct$, $H_x = \Ct$.
\begin{lemma}\label{lem:chain}
Let $x = (1, -1)$. Then $H_x = \Ct$.
\end{lemma}


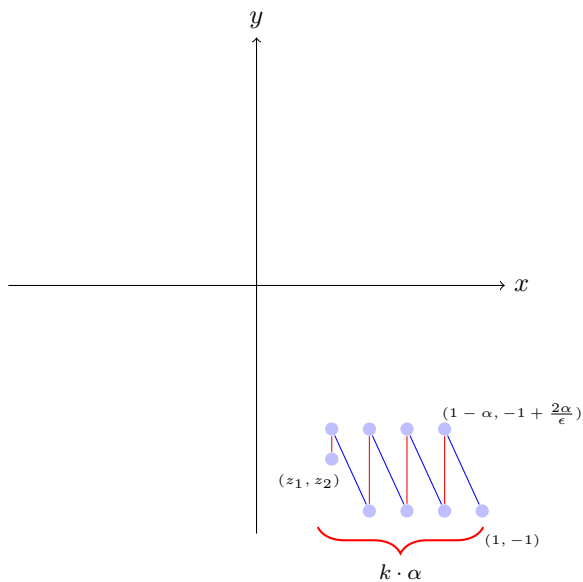
\begin{figure}[h]
\label{fig:zag}
\centering
\tikzstyle{vertex}=[circle,fill=blue!25,minimum size=5pt,inner sep=0pt]
\begin{tikzpicture}
      \draw[->] (-3.3,0) -- (3.3,0) node[right] {$x$};
      \draw[->] (0,-3.3) -- (0,3.3) node[above] {$y$};
      \node[vertex] (A) at (3,-3) {};
      \node[vertex] (B) at (2.5,-1.91) {};
      \node[vertex] (C) at (2.5,-3) {};
      \node[vertex] (D) at (2,-1.91) {};
      \node[vertex] (E) at (2,-3) {};
      \node[vertex] (F) at (1.5,-1.91) {};
      \node[vertex] (G) at (1.5,-3) {};
      \node[vertex] (H) at (1,-1.91) {};
      \node[vertex] (S) at (1,-2.31) {};
      \draw[-,mark=*,blue] (A) -- (B);
      \node at (3.4, -1.71)  {\tiny{$(1-\alpha, -1 +\frac{2\alpha}{\epsilon})$}};
      \node at (3.4, -3.4)  {\tiny{$(1, -1)$}};
      \node at (.7, -2.6)  {\tiny{$(z_1, z_2)$}};
      \draw[-,mark=*,red] (B) -- (C);
      \draw[-,mark=*,blue] (C) -- (D);
      \draw[-,mark=*,red] (D) -- (E);
      \draw[-,mark=*,blue] (E) -- (F);
      \draw[-,mark=*,red] (F) -- (G);
      \draw[-,mark=*,blue] (G) -- (H);
      \draw[-,mark=*,red] (H) -- (S);
\draw [thick, red,decorate,decoration={brace,amplitude=10pt,mirror},xshift=0.4pt,yshift=-0.4pt](0.8,-3.2) -- (3.0,-3.2) node[black,midway,yshift=-0.6cm] {\footnotesize $k \cdot \alpha$};
    \end{tikzpicture}
    \caption{A path connecting $(1,-1)$ to an arbitrary point
      $(z_1,z_2$): red segments are vertically equivalent, blue segments
      are diagonally equivalent.}\label{fig:equivalence}
\end{figure}

Lemma $\ref{lem:chain}$ shows that if $\A$ is fair at a given round $t$
with probability at least $\frac{3}{4}$ over the posterior, and the posterior
is $\prior$, then $\A$ must play uniformly at random from $\Ct$.

Thus, we have shown for any fair $\A$:
\begin{enumerate}
\item With probability at least $1-4\delta$, $\A$ must be fair with probability
at least $\frac{3}{4}$ at all $t \geq 1$ (Lemma ~\ref{lem:posterior-fair})
\item If $S$ is the number of rounds until $\prior \neq \prior|_{h_t}$,
$\mathbb{P}(S \geq t) \geq (1-2\epsilon)^{t}$
\item When $\prior = \prior|_{h_t}$ (i.e. $S \geq t$), and $\A$ is fair with
probability $> \frac{3}{4}$ over $\prior|_{h_t}$, then $\A$ must play uniformly
at random from $\Ct$
\end{enumerate}

Let $\epsilon < \min(1/2,1/\log(2/\delta))$ and let the event that $\A$ is fair with probability at least
$\frac{3}{4}$ over the posterior at all $t \geq 1$ be denoted by $F$. Recalling
that $S$ denotes the number of rounds required before the posterior
distribution of $\b_2$ becomes non-uniform, let
the event that $S \geq \frac{\log(1-\delta)}{\log(1-2\epsilon)}$ be denoted by
$E$. Then 
$$\Pr{}{E} \geq (1-2\epsilon)^{\tfrac{\log(1-\delta)}{\log(1-2\epsilon)}} =
1-\delta,$$
so
$$\Pr{}{E \cap F} \geq \Pr{}{E} + \Pr{}{F} - 1 \geq 1-5\delta.$$
We now condition on $F \cap E$ to show that with high probability $\regret{T} =
\tilde{\Omega}(\frac{1}{\epsilon})$:

\begin{equation}\label{eq:sne}
\begin{split}
\Pr{}{\regret{T} \geq \tilde{\Omega}\left(\frac{1}{\epsilon}\right)} \geq&\;
\Pr{}{\regret{T} \geq \tilde{\Omega}\left(\frac{1}{\epsilon}\right) \mid E \cap F}
\Pr{}{E \cap F} \\
\geq&\; \Pr{}{\regret{T} \geq \tilde{\Omega}\left(\frac{1}{\epsilon}\right) \mid E \cap F}
(1-5\delta)
\end{split}
\end{equation}

where the first inequality follows from Bayes' rule. However, we've shown that
whenever $E \cap F$ occurs, for at least
$\tfrac{\log(1-\delta)}{\log(1-2\epsilon)} \geq
\tfrac{\log(1/[1-\delta])}{2\epsilon}$ (via $\log(x) \leq x-1$ for $x > 0$) rounds
$\A$ plays uniformly at random from $\Ct$. Let $r_\A(t)$ be the regret
accrued at round $t$ by uniformly at random play,
$\Ex{}{r_\mathcal{A}(t)} = ||\b||_{1} = \Omega(1) = c$. Then
$0 \leq r_\mathcal{A}(t) \leq 2(1+\epsilon)$, and the $r_\A(t)$ are independent
since $\A$ is playing uniformly at random at each $t$. By Hoeffding's 
inequality for bounded random variables,
$$\Pr{}{\sum_{t=1}^T r_\A(t) \leq T \cdot c - 
\sqrt{2T\log(2/\delta)}(1+\epsilon)} \leq \delta$$
which means
\begin{equation}\label{eq:hoeff}
\Pr{}{\sum_{t = 1}^{T}r_\A(t) \geq T\cdot c - \sqrt{2T\log(2/\delta)}
(1+\epsilon)} \geq 1-\delta
\end{equation}
and when taking $T = \frac{1}{\epsilon}$ we get
$$\Pr{}{\sum_{t = 1}^{T}r_\A(t) \geq \frac{1}{\epsilon}\cdot c -
\sqrt{\frac{2}{\epsilon}\log(2/\delta)}(1+\epsilon)} \geq 1-\delta$$
or suppressing constants and lower order terms and using the fact that 
$\epsilon <1/\log(2/\delta)$, $\Pr{}{\sum_{t = 1}^{T}r_\A(t)
\geq \tilde{\Omega}(\frac{1}{\epsilon})} \geq 1-\delta$. This gives us that
$\Pr{}{\regret{T} \geq \tilde{\Omega}(\frac{1}{\epsilon}) \mid E \cap F} \geq
\frac{1-\delta}{1-5\delta}$. Hence by Equation~\ref{eq:sne}, $\Pr{}{\regret{T}
\geq  \tilde{\Omega}(1/\epsilon)} \geq 1-\delta$, as desired.

\end{proof}
We now provide the proofs of the lemmas used above. 
\begin{proof}[Proof of Lemma~\ref{lem:posterior-fair}]
By the definition of fairness, and Lemma ~\ref{lem:posterior}, we have that
$$\Pr{\b_t \sim \prior|_{h_t}, h_t \sim \A}{\exists t' \geq 1 \colon \A \text{\; is
round-unfair at time } t'} \leq \delta.$$
Denote this probability by $X$. By the above $\Ep[X] \leq \delta$, and hence by
Markov's inequality, $\Pr{}{X \geq \frac{1}{4}} \leq 4\delta$. But then we've
shown that, with probability at least $1-4\delta$, for all $t \geq 1$ $\A$ is
fair with probability at least $\frac{3}{4}$ over $\b \sim \prior|h_t$. Now if 
$\exists x,y$ such that $P_{\prior|h_t}(\vecdot{y'}{\b} > \vecdot{x'}{\b}) > \frac{1}{4}$ but $f_t(x) > f_t(y)$, then the probability that $\A$ is unfair at time $t$ is at least $\mathbb{P}_{\prior|h_t}(\vecdot{y'}{\b} > \vecdot{x'}{\b}) > \frac{1}{4}$. This proves the claim. 
\end{proof}

\begin{proof}[Proof of Lemma~\ref{lem:consistent}]
The fact that the posterior distribution of $\b_2$ is uniform on the set of
consistent $\b_2$ is immediate via Bayes rule:
$\prior(\b_2 |h_t) = p(h_t|\b_2)\prior(\b_2)$, where
$p(h_t|\b_2)\prior(\b_2) \propto 1$ if $\b_2$ is consistent with $h_t$, and is
$0$ otherwise.
\end{proof}

\begin{proof}[Proof of Lemma~\ref{lem:chain}]
Choose an arbitrary point $y \in \Ct$ with coordinates $(z_1, z_2)$. We want
  to show $y \in H_x$. Since any two points in $\Ct$ with the same $x$
  coordinate are vertically equivalent, it suffices to show that there
  is a point with $x$-coordinate $z_1 \in H_x$. 
  
  Fix $0 < \alpha \leq \min(1, 2\epsilon)$ and suppose $1-z_1 = k \cdot \alpha$,
  where $k \in \mathbb{N}$. Note we can guarantee $k \in \mathbb{N}$
  by choosing an appropriate $\alpha$. We now proceed by induction on
  $k$. 
  
  If $k = 1$, then by diagonal equivalence $x$ is equivalent to
  $x' = (1-\alpha, -1 + 2\alpha/\epsilon) = (z_1, 1 +
  2\alpha/\epsilon)$.
  But by vertical equivalence, $y \in H_{x'}$, and so $y \in H_x,$ by
  transitivity. For the inductive step, construct
  $x' = (z_1 + \alpha, z_2-2\alpha/\epsilon)$. Then
  $1-x'_1 = 1-z_1-\alpha = (k-1)\alpha$. Hence by induction
  $x' \in H_x$. But since $x'$ is diagonally equivalent to
  $y = (z_1, z_2)$, then $y \in H_x$ as desired. Since $y$ was
  arbitrarily chosen, $H_x = \Ct$. See Figure~\ref{fig:equivalence} for
  a visualization of these equivalences.
\end{proof}

\subsection{Proofs from Section~\ref{sec:lower_bound2}}

\begin{proof}[Proof of Theorem~\ref{thm:lb-two}]
  Let $E_{\b}$ be the event that given a fixed value of $\b$,
  $\A$ plays uniformly at random from $\Ct$ for all $t \geq 1$.
   If we
  can show that for any $\A$ and all $ \b$, it is the case
  that $\mathbb{P}(E_{\b}) = \Omega(1)$, this implies the claim,
  since for any $\b, T$
 \[\E{\regret{T}} \geq \E{\regret{T} \mid E_{\b}}
 \Pr{}{E_{\b}} = \Omega(T) \cdot  \Omega(1) = \Omega(T),\]
 as desired.

 By symmetry of $S^{1}$, $\Pr{}{E_{\b}} = \Pr{}{E_{\b'}}$ for
 all $\b, \b' \in S^{1}$. So henceforth we can drop the
 subscript $\b$, and use $E$ to represent the event that $\A$
 plays uniformly at random for all $t \geq 1$. We now exhibit a prior
 $\prior$ such that for any $\A$, $\Pr{}{E} = \Omega(1)$.

 Lemmas $\ref{lem:posterior}$ and $\ref{lem:posterior-fair}$ both
 apply; thus, we let $\b \sim \prior$, where $\prior$ is the
 uniform distribution on $S^{1}, U(S^{1})$, and we assume that at each
 time $t$, $\b$ is re-drawn from its posterior distribution
 $\prior|_{h_t}$, as before. Let $F$ again be the event that $\A$ is
 round-fair with probability at least $\frac{3}{4}$ at each round $t$,
 with respect to the posterior distribution $\prior|h_t$. We again
 analyze the posterior distribution $\prior|_{h_t}$, showing that for
 any history $h_t, \prior|_{h_t}$ forces $\A$ to play uniformly at
 random at $t$, conditioned on $F$.

 As in Section \ref{sec:lower_bound1} the posterior distribution of
 $\b|h_t$ is uniform on the set of $\b \in S^{1}$ that are
 consistent with the observed data. By consistent we again mean in the
 sense of Lemma $\ref{lem:consistent}$; the proof is nearly identical
 and relies on boundedness of the noise $\ett$, so we do not repeat
 it here. Denote by $G_t \subset S^{1}$ the set of consistent $\b$
 at time $t$.  We will use Lemma~\ref{lem:open} to reason about the
 topology of $G_t$. We use the relative topology throughout.

\begin{lemma}\label{lem:open}
  For any $t \geq 1$ and any history $h_t$, $G_t$ is a nonempty 
  connected open subset of $S^1$.
\end{lemma}

\begin{figure}[h]
\begin{center}
\includegraphics[width=0.9\textwidth]{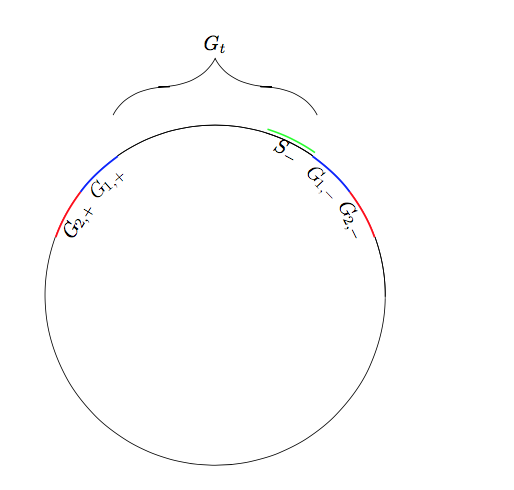}
\caption{$\A$ must play UAR from $\D = S^{1}$. $|G_t| = \epsilon$; $|S_{t, -}'| = |S_{t, -}| = \frac{3\epsilon}{4}$; $|S_{t,+}'| = |S_{t,+}| = |G_{1, +}| = |G_{1,-}| = |G_{2, +}| = |G_{2 , -}| = \epsilon$ }\label{fig:chaining}
\end{center}
\end{figure}

$G_t$ is an open, non-empty, connected subset of $S^{1}$; since we're working
in the relative topology, it must be exactly an open interval along the
boundary of $S^{1}$, as illustrated in Figure~\ref{fig:chaining}.
Let $G_t$ have length $\epsilon$, and correspondingly $\prior|_{h_t} = U(G_t)$.

Condition on the occurrence of $F$: that $\A$ must be fair in round $t$
with probability at least $\frac{3}{4}$, with respect to
$\prior|_{h_t}$. We claim that this in fact forces $\A$ to play
uniformly from $S^1$ at all time steps $t$, in an argument similar to
Lemma $\ref{lem:chain}$.

We say that two points $x, y \in S^{1}$ are equivalent at time $t$ if
$\Pr{\b \sim \prior|_{h_t}}{\vecdot{\b}{x} > \vecdot{\b}{y}} \in
[\frac{1}{4}, \frac{3}{4}]$.
Let $S_{x,t}$ be the transitive closure of the set of $y \in S^{1}$
that are equivalent to $x$ at time $t$.

\begin{lemma}\label{lem:chain2}
  Let $\prior|_{h_t} \sim U(G_t)$. Then there exists $x \in S^{1}$ such
  that $S_{x,t} = S^{1}$.
\end{lemma}
\begin{proof}[Proof of Lemma~\ref{lem:chain2}]
By definition, if $\Pr{\prior|h_t}{\vecdot{\b}{x} < \vecdot{\b}{y}}
\in [\frac{1}{4}, \frac{3}{4}]$, then $y \in S_{x,t}$. Every point on $S^1$ can be
represented as $(\cos \theta, \sin \theta)$, so let $\theta_x$ denote the angle
corresponding to $x$, and let $x$ be the point in $G_t$ such that 
$\Pr{\prior |h_t}{\b <  \theta_x} = \frac{1}{4}$. 

Now let $S_{t,-} = \{z \in G_t: \theta_z \geq \theta_x \}$ and let $S_{t,+} =
\{z \in G_t: \theta_z \leq \theta_x \}$. If $\b \in S_{t,+}$, then for all
$z \in S_{t,-}, \b \cdot z \leq \b \cdot x$. By construction,
$\Pr{\prior |h_t}{\b \in S_{t,+}} = \frac{1}{4}$, and hence $S_{t,-} \subset
G_{x,t}$. But defining $x_1$ as $\Pr{\prior |h_t}{\b > \theta_{x_1}} =
\frac{1}{4}$, $S_{t,+}'$ as the set $\{z \in G_t: \theta_z > \theta_{x_1} \}$, and
$S_{t, -}'$ as $\{z \in G_t: \theta_z < \theta_{x_1} \}$, the same reasoning
shows that $S_{t, -}' \subset S_{x_1, t}$. Since $S_{t,-} \cup S_{t,-}' = G_t$,
this forces $G_t \subset S_{x_1, t} = S_{x,t}$. 

We now show $S_{x,t}$ contains the rest of the boundary of $S^{1}$, not just
$G_t$. Let $G_{+}^{1}$ denote the arc of length $\frac{1}{4}\epsilon$ adjoining
$S_{t,+}'$ as in Figure~\ref{fig:chaining}, and define $G_{-}^{1}$ accordingly.
Now note that we must have $G_{+}^{1} \in G_{x,t}$, since if $\b >
\theta_{x_1}$ then for all $z \in G_{+}^{1}, \b \cdot z > \b \cdot x$,
and $\Pr{\prior |h_t}{\b > \theta_{x_1}} = \frac{1}{4}$. Similarly,
$G_{-}^{1}$ has to be added to $S_{x_1,t} = S_{x,t}$ as well. But then letting
the segment $G_{-}^{1} \cup G_{+}^{1} \cup G_t$ be denoted by $G_t'$, we can
repeat the argument: we set $x', x_1'$ to be their initial locations $x_1, x$ 
translated $\frac{1}{4}\epsilon$ to the right and left respectively, and define
$G_{+}^{2}, G_{-}^{2}$ analogously, as in the Figure~\ref{fig:chaining}. 

Now we have that $G_{+}^{2} \in S_{x',t}$, since if $\b \in S_{t,+}'$ then for
all $z \in G_{+}^{2}, \b \cdot z > \b \cdot x'$, and hence $z \in
S_{x',t} = S_{x,t}$. The same logic shows that $G_{+}^{2} \subset S_{x_1',t} = S_{x,t}$.

Since we can keep recursively chaining segments of fixed length
$\frac{\epsilon}{4}$ to $S_{x,t}$, and $S^{1}$ is of fixed length, a simple 
induction argument forces $S_{x,t} = S^{1}$, as desired.
\end{proof}

So Lemma \ref{lem:posterior-fair} in combination with the above lemma
forces the following: when $\A$ is constrained to be fair with probability at
least $\frac{3}{4}$ with respect to the posterior distribution of
$\b$, for all times $t \geq 1$ and all histories $h_t$, $\A$ must
play uniformly at random from $S^{1}$. But then
$\mathbb{P}(E) \geq \mathbb{P}(E|F)\mathbb{P}(F) = \mathbb{P}(F) \geq
1-4\delta = \Omega(1)$,
by Lemma $\ref{lem:posterior-fair}$.
\end{proof}

\begin{proof}[Proof of Lemma~\ref{lem:open}]
  $C_t \neq \varnothing$ is immediate since, for the true value
  $\b$, $\b \in C_t$ for all $t$. For $\b \in S^{1}$ to be
  consistent with the data, i.e. in $C_t$, means that
  $\max_{1 \leq i \leq t} |y_i - \vecdot{\b}{x_i}| < 1$ and
  $\b \in S^{1}$. 
  
  We can rephrase this as follows: if
  $f_i(\b) = |y - \vecdot{\b}{x_i}|$, and
  $R_i = \{\b \in f_i^{-1}(-\infty, 1)\},$ then if we let
  $C_t' = \bigcap_{i = 1}^{t}R_i$, $C_t = C_t' \cap S^{1}$. Now we
  remark that each $R_i$ is the intersection of the two open half
  spaces $\{ \b: \vecdot{\b}{x_i} < 1 + y_i\}$ and
  $\{ \b: \vecdot{\b}{x_i} > y_i-1\}.$ Thus $C_t'$ is the
  intersection of finitely many open half spaces, and is thus an open,
  connected set (in fact, it is a convex polytope). Since
  $C_t = S^{1} \cap C_t'$, by definition $C_t$ is open and connected
  in the relative topology on $S^{1}$.
\end{proof}

\subsection{Experiments}\label{subsec:experiments}
Figure~\ref{fig:both_mistreatment} depicts experiments conducted in the \UB setting.
We employ a simple variant of UCB that maintains generic normal confidence 
intervals around its ongoing estimate of $\beta$ and uses these to construct
confidence intervals for the estimated rewards of the contexts is uses; it then
selects all choices with a positive upper confidence bound. We plot
\emph{cumulative mistreatments} through $T = 10,000$ rounds, which tracks the
cumulative number of individuals who have seen an individual with lower
expected quality chosen in a round during which they were not chosen. The plot
therefore shows that through 10,000 rounds our version of UCB creates nearly
400 such mistreated people.

Our experiments use $d = 2$ and $\beta \sim U[-1,1]^2$ for each 
iteration. In each round we generate $k = 10$ contexts $x_i$, also from 
$U[-1,1]^2$, and generate noisy rewards $\beta \cdot x_i + \eta_{t,i}$
where $\eta_{t,i} \sim N(0,1)$ is standard normal noise. The results presented
are averaged over 100 iterations. For completeness, we present
Figure~\ref{fig:both_mistreatment}, which plots cumulative mistreatments for both UCB
and FairUCB and empirically validates our theoretical fairness guarantee.

\vspace{10pt}

\begin{center}
\begin{figure}[h]
\begin{center}
\vspace{-20pt}
\includegraphics[width=0.6\textwidth]
{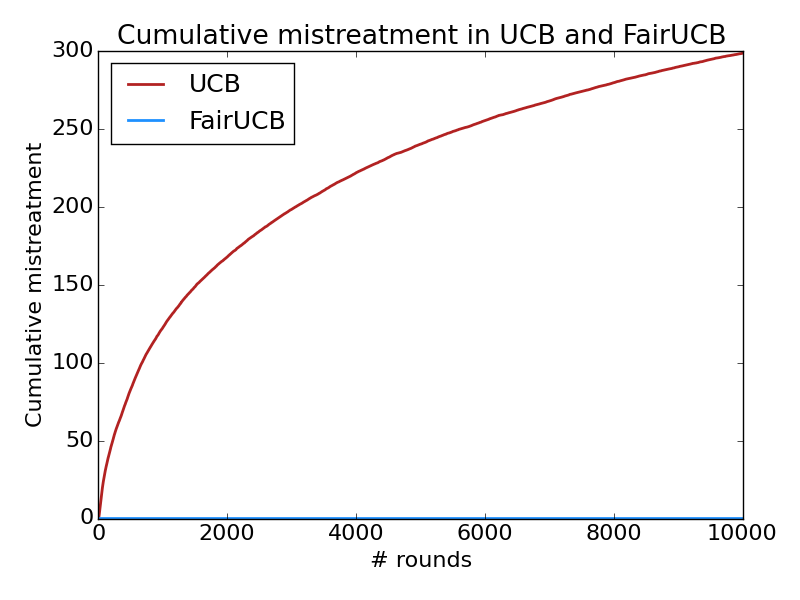}
\caption{Cumulative mistreatments for UCB and FairUCB.}
\end{center}
\label{fig:both_mistreatment}
\end{figure}
\end{center}

Our second experiment investigates the \emph{structure} of mistreatment in UCB. 
We use $d = 2, \beta = [1,0], k = 10$ for each iteration. At each round $t$ with
probability $p \in [.8, .95]$ we draw a context $(x, x),$ where $x \sim$
$U[-1,1]$ and with probability $1-p$ draw a context from $U[-1,1]^2$.
These two types of contexts naturally encode two populations: in population 1,
the two features are perfectly correlated and in population 2 they are
independent. However, $\beta = [1,0]$ crucially means that the second feature
\emph{does not affect reward}. Our experiments aim to study how this
correlation affects mistreatment rates in the different populations.
 
For each population we plot the fraction of mistreatment individuals from each
population for $T = 1, \ldots 25,$ averaging over $1000$ iterations.
Figure~\ref{fig:disparity} shows that for $p \in [.8,.95]$ unfairness accrues
at substantially different rates to the two populations. Somewhat
counter-intuitively, members of the majority group are significantly more
likely to be mistreated than members of the minority group, a natural 
consequence of UCB-style algorithms favoring minority contexts whose confidence
intervals have more uncertainty. While mistreating a majority population may be
less obviously unfair than mistreating a minority population, it is still
undesirable. In particular, there may be natural practical settings where the
group that has faced historical discrimination is the majority population in
sample (e.g. criminal sentencing) and so discriminating against the majority
is more obviously unfair. 

\vspace{10pt}
\begin{center}
\begin{figure}[h]
\begin{center}
\vspace{-20pt}
\includegraphics[width=.6\textwidth]
{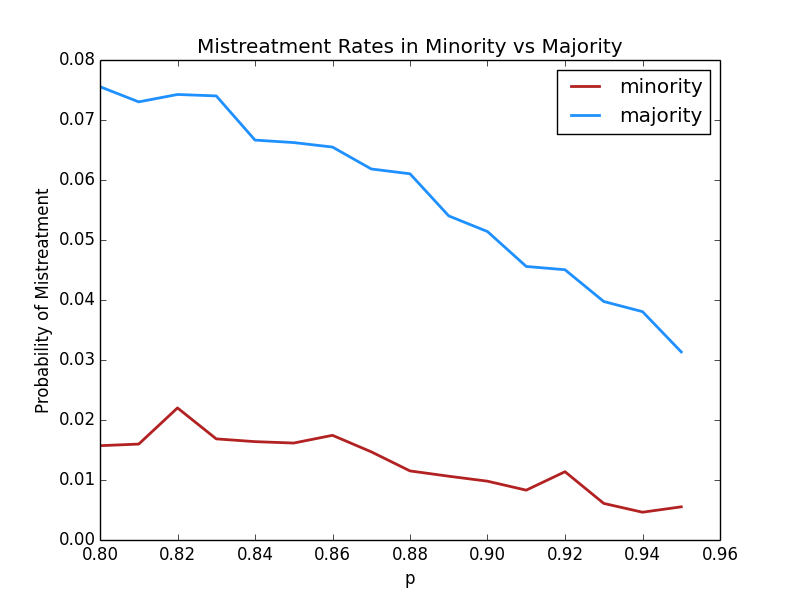}
\caption{Probability of mistreatment for subpopulations under UCB.}
\end{center}
\label{fig:disparity}
\end{figure}
\end{center}

\end{document}